\documentclass{article}

\PassOptionsToPackage{numbers, sort&compress}{natbib}

\usepackage[preprint]{neurips_2019}

\usepackage[utf8]{inputenc} 
\usepackage[T1]{fontenc}    
\usepackage{hyperref}       
\usepackage{url}            
\usepackage{booktabs}       
\usepackage{amsfonts}       
\usepackage{nicefrac}       
\usepackage{microtype}      
\usepackage{authblk}

\title{Near-Optimal Online Egalitarian learning in General Sum Repeated Matrix Games}
\usepackage{times}
\author[1]{Aristide Tossou}
\author[1]{Christos Dimitrakakis}
\author[2]{Jaroslaw Rzepecki}
\author[2]{Katja Hofmann}
\affil[1]{Chalmers University of Technology, G{\"o}teborg, Sweden}
\affil[2]{Microsoft Research Lab, Cambridge, UK}
\affil[ ]{\textit{\{aristide,chrdimi\}@chalmers.se; \{jaroslaw.rzepecki,katja.hofmann\}@microsoft.com}}

\usepackage{microtype}
\usepackage{clipboard}
\usepackage{booktabs}
\usepackage{amsmath}
\usepackage{amsmath}
\usepackage{amssymb}
\usepackage{amsthm}
\usepackage{mathabx}
\usepackage{bm}
\usepackage{algorithm}
\usepackage{algorithmicx}
\usepackage{algpseudocode}
\usepackage{mathtools}

\usepackage{xargs}                      
\usepackage{xcolor}  
\usepackage{subcaption}
\usepackage{float}
\usepackage{tabulary}
\usepackage{csquotes}
\usepackage{filecontents}

\usepackage[colorinlistoftodos,prependcaption,textsize=tiny]{todonotes}
\newcommandx{\unsure}[2][1=]{\todo[linecolor=red,backgroundcolor=red!25,bordercolor=red,#1]{#2}}
\newcommandx{\change}[2][1=]{\todo[linecolor=blue,backgroundcolor=blue!25,bordercolor=blue,#1]{#2}}
\newcommandx{\info}[2][1=]{\todo[linecolor=OliveGreen,backgroundcolor=OliveGreen!25,bordercolor=OliveGreen,#1]{#2}}
\newcommandx{\improvement}[2][1=]{\todo[linecolor=Plum,backgroundcolor=Plum!25,bordercolor=Plum,#1]{#2}}
\newcommandx{\thiswillnotshow}[2][1=]{\todo[disable,#1]{#2}}

\DeclareMathOperator*{\argmax}{argmax}
\DeclareMathOperator*{\argmin}{argmin}
\DeclareMathOperator{\E}{\mathbb{E}}
\DeclareMathOperator{\Prob}{\mathbb{P}}
\DeclareMathOperator{\Id}{\mathbb{I}}
\DeclareMathOperator{\BigO}{\mathcal{O}}
\DeclareMathOperator*{\geql}{\geq_{\ell}}

\DeclarePairedDelimiter\abs{\lvert}{\rvert}%

\DeclarePairedDelimiter\size{\lvert}{\rvert}%
\DeclarePairedDelimiter\setof\{\}
\DeclarePairedDelimiter\curly\{\}
\DeclarePairedDelimiter\paren()%

\DeclarePairedDelimiterX{\klx}[2]{(}{)}{%
  #1\;\delimsize\|\;#2%
}
\newcommand{\kl}{\text{KL}\klx}

\newcommand{\piEg}[1][]{%
	\ifthenelse{\equal{#1}{}}{\pi_{\scriptscriptstyle\text{Eg}}}{\pi_{\scriptscriptstyle#1,\text{Eg}}}%
}
\newcommand{\SV}{\text{S}\!V}
\newcommand{\SVUp}{\text{S}\!\hat{V}}
\newcommand{\SVDown}{\text{S}\!\check{V}}
\newcommand{\Eg}{\text{Eg}}
\newcommand{\piUp}{\hat{\pi}}
\newcommand{\piDown}{\check{\pi}}
\newcommand{\aUp}{\hat{a}}

\newcommand{\piSV}[1][]{%
	\ifthenelse{\equal{#1}{}}{\pi_{\scriptscriptstyle\SV}}{\pi_{\scriptscriptstyle#1,\SV}}%
}
\newcommand{\piBR}[1][]{%
	\ifthenelse{\equal{#1}{}}{\pi_{\scriptscriptstyle\text{BR}}}{\pi_{\scriptscriptstyle#1,\scriptscriptstyle\text{BR}}}%
}

\newcommand{\piBounded}[1][]{%
	\ifthenelse{\equal{#1}{}}{\pi_{\scriptscriptstyle\text{BM}}}{\pi_{\scriptscriptstyle#1,\text{BM}}}%
}

\newcommand{\piMinimax}[1][]{%
	\ifthenelse{\equal{#1}{}}{\pi}{\pi_{#1,}}%
}
\newcommand{\VEg}[1][]{%
	\ifthenelse{\equal{#1}{}}{V_{\scriptscriptstyle\text{Eg}}}{V_{\scriptscriptstyle#1,\text{Eg}}}%
}
\newcommand{\VEgR}[1][]{%
	\ifthenelse{\equal{#1}{}}{V_{\scriptscriptstyle\text{Eg}}}{V_{\scriptscriptstyle\text{Eg},#1}}%
}
\newcommand{\VEgUp}[1][]{%
	\ifthenelse{\equal{#1}{}}{\hat{V}_{\scriptscriptstyle\text{Eg}}}{\hat{V}_{\scriptscriptstyle#1,\text{Eg}}}%
}
\newcommand{\VBR}[1][]{%
	\ifthenelse{\equal{#1}{}}{V_{\scriptscriptstyle\text{BR}}}{V_{\scriptscriptstyle#1,\text{BR}}}%
}
\newcommand{\VLambda}[1][]{%
	\ifthenelse{\equal{#1}{}}{V_{\Lambda}}{V_{\scriptscriptstyle#1,\Lambda}}%
}

\newcommand{\GEg}[1][]{%
	\ifthenelse{\equal{#1}{}}{G_{\scriptscriptstyle\text{Eg}}}{G_{\scriptscriptstyle#1,\text{Eg}}}%
}
\newcommand{\GLambda}[1][]{%
	\ifthenelse{\equal{#1}{}}{G_{\Lambda}}{G_{\scriptscriptstyle#1,\Lambda}}%
}

\newcommand{\rUp}{\hat{r}}
\newcommand{\rEmp}{\bar{r}}
\newcommand{\MUp}{\hat{M}}
\newcommand{\MDown}{\check{M}}
\newcommand{\rDown}{\check{r}}
\newcommand{\VUp}{\hat{V}}
\newcommand{\VDown}{\check{V}}
\newcommand{\numact}{\size{\mathcal{A}}}

\newcommand{\regret}{\textrm{Regret}}
\newcommand{\challenge}{\emph{small $\epsilon$-error large regret}}

\newcommand{\IndState}[1][1]{\State\hspace{\algorithmicindent}}
\makeatletter
\newcommand{\algmargin}{\the\ALG@thistlm}
\makeatother
\algnewcommand{\ParState}[1]{\State%
	\parbox[t]{\dimexpr\linewidth-\algmargin}{\strut\hangindent=\algorithmicindent \hangafter=1 #1\strut}}

\makeatletter
\newcommand{\@giventhatstar}[2]{\left(#1\;\middle|\;#2\right)}
\newcommand{\@giventhatnostar}[3][]{#1(#2\;#1|\;#3#1)}
\newcommand{\giventhat}{\@ifstar\@giventhatstar\@giventhatnostar}
\makeatother

\algnewcommand\algorithmicforeach{\textbf{for each}}
\algdef{S}[FOR]{ForEach}[1]{\algorithmicforeach\ #1\ \algorithmicdo}

\newcommand{\epoch}{{epoch}}
\newcommand{\epochs}{{epochs}}

\newcommand{\round}{{round}}
\newcommand{\rounds}{{rounds}}

\newcommand{\minimax}{{maximin}}
\newcommand{\Minimax}{{Maximin}}

\colorlet{dark-cyan}{cyan!75!black}

\newtheorem{theorem}{Theorem}
\newtheorem{lemma}{Lemma}
\newtheorem{fact}{Fact}
\newtheorem{definition}{Definition}
\theoremstyle{definition}
\newtheorem{example}{Example}
\theoremstyle{definition}

\newtheorem{proposition}{Proposition}
\newtheorem*{theorem*}{Theorem}

\newtheorem*{proposition*}{Proposition}
\newtheorem*{fact*}{Fact}

\algdef{SE}[DOWHILE]{Do}{DoWhile}{\algorithmicdo}[1]{\algorithmicwhile\ #1}%

\usepackage{xr}
\makeatletter
\newcommand*{\addFileDependency}[1]{
  \typeout{(#1)}
  \@addtofilelist{#1}
  \IfFileExists{#1}{}{\typeout{No file #1.}}
}
\makeatother

\begin{document}

\maketitle

\begin{abstract}  
We study two-player general sum repeated finite games where the rewards of each player are generated from an unknown distribution. Our aim is to find the egalitarian bargaining solution (EBS) for the repeated game, which can lead to much higher rewards than the \minimax{} value of both players. Our most important contribution is the derivation of an algorithm that achieves simultaneously, for both players, a high-probability regret bound of order $\BigO\paren*{\sqrt[3]{\ln T}\cdot T^{2/3}}$ after any $T$ rounds of play. We demonstrate that our upper bound is nearly optimal by proving a lower bound of $\Omega\paren*{T^{2/3}}$ for any algorithm. 
\end{abstract}
%
%




\section{Introduction}
\label{sec:introduction}

Multi-agent systems are ubiquitous in many real life applications such as self-driving cars, games, computer networks, etc. 
Agents acting in such systems are usually self-interested and aim to maximize their own individual utility. To achieve their best utility, agents face three key fundamental questions: exploit, cooperate, or insure safety? 
Learning how agents should behave when faced with these significant challenges is the subject of this paper. 
We focus on two players (one is called agent, the other opponent) repeated games, a setting which captures the key challenges faced when interacting in a multi-agent system where at each \round{}, the players simultaneously select an action and observe an individual numerical value called reward. The goal of each player in this game, is to maximize the sum of accumulated rewards over many \rounds{}.
One of the key dilemmas for learning in repeated games is the lack of a single optimal behavior that is satisfactory against all opponents, since the best strategy necessarily depends on the opponent. 


\citet{Powers2007} tackle this dilemma and propose a rigorous criterion called \emph{guarded optimality} 
which for two players simplifies to three criteria: (1) Targeted Optimality: when the opponent is a member of the target set, the
average reward is close to the best response against that opponent;
(2) Safety: against any opponent, the average reward is close to a safety value; (3) Individual Rationality: in self-play, the average reward is Pareto efficient\footnote{i.e., it is impossible for one agent to change to a better policy without making the other agent worse off.} and individually not below the safety value.

In this paper, we adopt those criteria and focus on the self-play settings. We pick the safety value to be the largest value one can guarantee against any opponent (also called \minimax{} value, see Definition \ref{def:minimax}). For the individual rationality criterion, we depart from previous works by considering the so called egalitarian bargaining solution (EBS) \citep{kalai1977proportional} in which both players bargain to get an equal amount above their \minimax{} value. This EBS is a Nash equilibrium (NE)
for the repeated game, a direct consequence of the \emph{folk theorems} \citep{osborne1994course}
and in many games (see Example \ref{example:games})
 has a value typically no worse for both players than values achievable by single-stage (i.e. non-repeated) NE usually considered in the literature. We pick the EBS since it satisfies even more desirable properties \citep{kalai1977proportional} on top of the \emph{individual rationality} criterion such as: independence of irrelevant alternatives (i.e. eliminating choices that were irrelevant does not change the choices of the agents), individual monotonicity (a player with better options should get a weakly-better value)
and (importantly) uniqueness. 
It is also connected to fairness and  \citet{rawls2009theory} theory of justice for human society \citep{kalai1977proportional}.

\paragraph{Related work}

Our work is related to \citet{poly_equilibrium_stochastic} where an algorithm to find the same egalitarian solution for general-sum repeated stochastic games is provided. When applied to general-sum repeated games, their algorithm implies finding an approximate solution using a (binary) search through the space of policy. Instead, our result will find the exact egalitarian solution with a more direct and simple formula. Also \citet{poly_equilibrium_stochastic} and many other works such as \citep{poly_equilibrium_bimatrix, Powers2007, chakraborty2014multiagent} assume deterministic rewards known to both players. In this work, we consider the case of stochastic rewards generated from a fixed distribution unknown by both players.

Another difference with many previous works is the type of solution considered in self-play. Indeed, we consider a NE for the repeated game whereas works such as \citep{chakraborty2014multiagent, conitzer2007awesome, banerjee2004performance, powers2005learning} consider the single-stage NE. The single-stage NE is typically undesirable in self-play since equilibrium with much higher values can be achieved as illustrated by Example \ref{example:games} in this paper. Other works such as \cite{Powers2007} consider optimizing for the sum of rewards in self-play. However, as illustrated by Example \ref{example:games} in this paper, this sum of rewards does not always guarantee individual rationality since some player could get lower than their \minimax{}.

\citet{crandall2011learning, stimpson2003learning} proposes algorithms with the goal of converging to a NE of the repeated games. However, \cite{crandall2011learning} only show an asymptotic convergence empirically in a few games while \cite{stimpson2003learning} only show that some parameters of their algorithms are more likely to converge asymptotically. Instead, we provide finite-time theoretical guarantees for our algorithm. Although in their settings players only observe their own rewards and not the other player, they assume deterministic rewards.

\citet{NIPS2017_7084, brafman2002r} tackles online learning for a generalization of repeated games called \emph{stochastic games}. However, they consider zero-sum games where the sum of the rewards of both players for any joint-action is always 0. In our case, we look at the general sum case where no such restrictions are placed on the rewards. In the learning settings there are other single-stage equilibrium considered such as correlated-equilibrium \citep{correlatedq}.

Our work is also related to multi-objective multi-armed bandit \citep{drugan2013designing} by considering the joint-actions as arms controlled by a single-player.
Typical work consider on multi-objective multi-armed bandit tries to find any solution that minimizes the distance between the Pareto frontier. However, not all Pareto efficient solutions are acceptable as illustrated by Example \ref{example:games} in this paper. Instead, our work show that a specific Pareto efficient (the egalitarian) is more desirable.

\paragraph{Paper organization}
The paper is organized as follows: Section \ref{sec:settings} presents formally our setting, assumptions, as well key definitions needed to understand the remainder of the paper. Section \ref{sec:algorithms} shows a description of our algorithm while section \ref{sec:theory} contains its analysis as well as the lower bound. We conclude in section \ref{sec:conclusion} with indication about future works. Detailed proof of our theorems is available in the appendix.

\section{Background and Problem Statement}
\label{sec:settings}
We focus on two-player general sum repeated games. 
At round $t$, both players select and play a joint action $a_t = (a^i_t, a^{-i}_t)$ from a finite set $\mathcal{A} = \mathcal{A}^{i} \times \mathcal{A}^{-i}$. Then, they  receive rewards $(r_t^{i}, r_t^{-i}) \in [0,1]^2$ generated from a fixed but unknown bounded distribution depending on their joint action. The actions and rewards are then revealed to both players.
%
%
We assume the first agent to be under our control and the second agent to be the opponent. We would like to design algorithms such that our agent's cumulative rewards are as high as possible. The opponent can have one of two types known to our agent: (1) \emph{self-player} (another independently run version of our algorithm) or (2) arbitrary (i.e any possible opponents with no access to the agent's internal randomness).

To measure performance, we compare our agent to an oracle that has full knowledge of the distribution of rewards for all joint-actions. The oracle then plays like this: (1) in self-play, they both compute before the game start the egalitarian equilibrium and play it; (2) against any other arbitrary opponent, the oracle plays the policy ensuring the \minimax{} value.

Our goal is to design algorithms that have low expected regret against this oracle after any number of $T$ rounds, where regret is the difference between the value that the oracle would have obtained and the value that our algorithm actually obtained. Next, we formally define the terms that describe our problem setting.

\begin{definition}[Policy]
\label{def:policy}
A policy $\pi^i$ in a repeated game for player $i$ is a mapping from each possible history to a distribution over its actions. That is: $\forall t \geq 0, \pi^i: \mathcal{H}_t \to \Delta \mathcal{A}^i$ where $t$ is the current \round{} and $\mathcal{H}_t$ is the set of all possible history of joint-actions up to \round{} $t$.

A policy is called \emph{stationary} if it plays the same distribution at each \round{}. 
It is called \emph{deterministic stationary} if it plays the same action at each \round{}.
\end{definition}

\begin{definition}[Joint-Policy]
\label{def:joint_policy}
A joint policy $(\pi^{i}, \pi^{-i})$ is a pair of policies, one for each player $i, -i$ in the game. In particular, this means that the probability distributions over actions of both players are independent.  When each component policy is \emph{stationary}, we call the resulting policy \emph{stationary} and similarly for \emph{deterministic stationary}.
\end{definition}

\begin{definition}[Correlated-Policy]
\label{def:correlated_policy}
Any joint-policy where player actions are not independent is \emph{correlated}\footnote{For example through a public signal.}.
A correlated policy $\pi$ specifies a probability distribution over joint-actions known by both players: $\forall t \geq 0, \pi: \mathcal{H}_t \to \Delta \mathcal{A}$.
\end{definition}
In this paper, when we refer to a policy $\pi$ without any qualifier, we will mean a correlated-policy, which is required for the egalitarian solution. When we refer to  $\pi^i$ and $(\pi^i, \pi^{-i})$ we will mean the components of a non-correlated joint-policy. 

\subsection{Solution concepts}
In this section, we explain the two solution concepts we aim to  address:  safety--selected as the \minimax{} value and individual rationality selected as achieving the value of the EBS. We start from the definition of a value for a policy.

\begin{definition}[Value of a policy]
	\label{def:value}
	The value $V^i(\pi)$ of a policy $\pi$ for player $i$ in a repeated game $M$ is defined as the infinite horizon undiscounted expected average reward given by:
	\begin{align*}
		V_M^i(\pi) &= \limsup_{T\to \infty} \frac{1}{T}  \E \left(\sum_{t=1}^{T} r_{t}^i ~\middle|~ \pi, M \right).
	\end{align*}
     We use $V_M = (V_M^i, V_M^{-i})$ to denote values for both players and drop $M$ when clear from the context.
\end{definition}

\begin{definition}[\Minimax{} value]
	


The \minimax{} policy $\piSV^i$ for player $i$ and its value $\SV^i$ are such that:
\begin{align*}
\piSV^i &= \argmax_{\pi^i} \min_{\pi^{-i}} V^i(\pi^i, \pi^{-i}), \quad \quad \SV^i = \max_{\pi^i} \min_{\pi^{-i}} V^i(\pi^i, \pi^{-i}).
\end{align*} where $V^i(\pi^i, \pi^{-i})$ is the value for player $i$ playing policy $\pi^i$ while all other players play $\pi^{-i}$.
  \label{def:minimax}
  
  
\end{definition}

\begin{definition}[Advantage game and Advantage value]
Consider a repeated game between two players $i$ and $-i$ defined by the joint-actions $\mathcal{A} = \mathcal{A}^i \times \mathcal{A}^{-i}$ and the random rewards $r$ drawn from a distribution $R: \mathcal{A} \to \Delta \mathbb{R}^2$. Let $\SV = (\SV^i, \SV^{-i})$ be the \minimax{} value of the two players. The \emph{advantage game} is the game with (random) rewards $r_+$ obtained by subtracting the \minimax{} value of the players from $r$. More precisely, the advantage game is defined by: $r_+(a) = r(a) - \SV \; \forall a \in \mathcal{A}$.
The value of any policy in this advantage game is called advantage value.
\end{definition}

\begin{definition}[EBS in repeated games]
\label{definition:equilibrium}
Consider a repeated game between two players $i$ and $-i$ with \minimax{} value $\SV = (\SV^i, \SV^{-i})$. A policy $\piEg$ is an EBS if it satisfies the following two conditions: (1) it belongs to the set $\Pi_\Eg$ of policies maximizing the minimum of the advantage value for both players. (2) it maximizes the value of the player with highest advantage value.

More formally, for any vector $x=(x^1, x^2) \in \mathbb{R}^2$, let $L: \mathbb{R}^2 \to \mathbb{R}^2$ be a permutation of $x$ such that $L^1(x) \leq L^2(x)$. Let's define a lexicographic maximin ordering $\geql$ on $\mathbb{R}^2$ as:
\[x \geql y \iff \paren*{L^1(x) > L^1(y)} \lor \paren*{L^1(x) = L^1(y) \land L^2(x) \geq L^2(y)} \quad \text{for any}\; x \in \mathcal{R}^2, y \in \mathcal{R}^2\]
A policy $\pi_{\Eg}$ is an EBS \footnote{Also corresponds to the leximin solution to the Bargaining problem \cite{bossert1995arbitration}.} if: 
${V(\pi_{\Eg}) - \SV} \geql {V(\pi)-\SV} \; \forall \pi$

We call EBS value the value $\VEg[] = V(\pi_{\Eg})$ and $V_+(\pi_{\Eg}) = V(\pi_{\Eg}) -\SV$ will be used to designate the egalitarian advantage.
\end{definition}

\subsection{Performance criteria}
We can now define precisely the two criteria we aim to optimize.

\begin{definition}[Safety Regret]
\label{def:regret_safe}
The safety regret for an algorithm $\Lambda$ playing for $T$ \rounds{} as agent $i$ against an arbitrary opponent $\piMinimax[]^{-i}$ with no knowledge of the internal randomness of $\Lambda$ is defined by:\[\regret_T(\Lambda, \piMinimax[]^{-i}) = \sum_{t=1}^{T} \SV^i - r^i_{t}\]
\end{definition}
\begin{definition}[Individual Rational Regret]
\label{def:regret_rational}
The individual rational regret for an algorithm $\Lambda$ playing for $T$ \rounds{} as agent $i$ against its self $\Lambda'$ identified as $-i$ is defined by:
\begin{align*}
\regret_T(\Lambda, \Lambda') &= \max\curly*{\sum_{t=1}^{T} \VEg[]^i - r_{t}^i, \sum_{t=1}^{T} \VEg[]^{-i} - r_{t}^{-i}}
\end{align*}
\end{definition}

\begin{example}[Comparison of the EBS value to other concepts]
\label{example:games}
In Table \ref{table:example}, we present a game and give the values achieved by the single-stage \emph{NE}, and Correlated Equilibrium \cite{correlatedq} (\emph{Correlated}); maximizing the sum of rewards (\emph{Sum}), and a Pareto-efficient solution (\emph{Pareto}). In this game, the \minimax{} value is $(\frac{3}{10},\frac{3}{10})$. \emph{Sum} leads to $\frac{1}{10}$ for the first player, much lower than its \minimax{}. \emph{Pareto} is also similarly problematic. Consequently, it is not enough to converge to any Pareto solution since that does not necessarily guarantee rationality for both players. Both \emph{NE} and \emph{Correlated} fail to give the players a value higher than their \minimax{} while the EBS shows that a high value $(\frac{23}{25}, \frac{23}{25})$ is achievable. A conclusion similar to this example can also be made for all non trivial zero-sum games.
\end{example}
\begin{table}[H]
\renewcommand{\arraystretch}{1.3}
	\centering
	\begin{subtable}{0.25\textwidth}%
	\begin{tabulary}{\textwidth}{C|C|C|}
	& C & D  \\ \hline
	C & $\frac{4}{5}, \frac{4}{5}$	& $\frac{1}{10}, \frac{9}{5}$\\ \hline
	D & $\frac{9}{5}, 0$	& $\frac{3}{10}, \frac{3}{10}$   \\ \hline
	\end{tabulary} %
	\caption{Game}
	\end{subtable}%
\begin{subtable}{0.75\textwidth}
\begin{tabulary}{\textwidth}{|C|C|C|C|C|C|}
	\hline
	\Minimax{} & Egalitarian & Nash & Sum & Correlated & Pareto  \\ \hline
	$\frac{3}{10}, \frac{3}{10}$ & $\frac{23}{25}, \frac{23}{25}$	& $\frac{3}{10}, \frac{3}{10}$ & $\frac{1}{10}, \frac{9}{5}$ & $\frac{3}{10}, \frac{3}{10}$ &  $\frac{9}{5}, 0$\\ \hline
	\end{tabulary} %
	\caption{Comparison of solutions}
	\end{subtable} %
			\caption{Comparison of the EBS to others concepts}
	\label{table:example}
\end{table}


\section{Methods Description}
\label{sec:algorithms}
\paragraph{Generic structure}
\label{sec:structure}
Before we detail the safe and individual rational algorithms, we will describe their general structure.
The key challenge is how to deal with uncertainty, the fact that we do not know the rewards.
To deal with this uncertainty, we use the standard principle of \emph{optimism in the face of uncertainty} \cite{jaksch2010near}. It works by \textbf{a)} constructing a set of \emph{statistically plausible games} containing the true game with high probability through a
confidence region around estimated mean rewards, a step detailed in section \ref{sec:plausible_sets}; \textbf{b)} finding within that set of
plausible games the one whose EBS policy (called \emph{optimistic}) has the highest value, a step detailed in section \ref{sec:optimistic_ebs_policy}; \textbf{c)} playing this optimistic policy until the start of an artificial \emph{\epoch{}} where
a new epoch starts when the number of times any joint-action has been played is doubled (also
known as the \emph{doubling trick}), a step described in \citet{jaksch2010near} and summarized by Algorithm \ref{algo:egalitarian_optimism} in Appendix \ref{sec:appendix_algo}.

\subsection{Construction of the plausible set}
\label{sec:plausible_sets}
At \epoch{} $k$, our construction is based on creating a set $\mathcal{M}_k$ containing all possible games with expected rewards $\E r$ such that,
\begin{equation}
\mathcal{M}_k = \{r : \abs{\E r^i(a) - \bar{r}_{k}^i(a)} \leq C_{k}(a) \forall i, a\},\qquad
C_{k}(a) = \sqrt{\frac{2 \ln 1/\delta_{k}}{N_{t_k}(a)}},
\label{plausible_set}
\end{equation}
where
$N_{t_k}(a)$ is the number of times action $a$ has been played up to \round{} $t_k$, $\bar{r}_{k}(a)$ is the empirical mean reward observed up to \round{} $t_k$ and $\delta_{k}$ is an adjustable probability.
 %
The plausible set can be used to define the following upper and lower bounds on the rewards of the game:
\[
\rUp^i_{k}(a) = \bar{r}_k^i(a) + C_k(a),
\qquad
\rDown^i_k(a) = \bar{r}_k^i(a) - C_k(a).
\]
We denote $\MUp$ the game with rewards $\rUp$ and $\MDown$ the game with $\rDown$. Values in those two games are resp. denoted $\VUp$, $\VDown$.
We used $C_{k}(\pi)$, $C_{k}(\pi^{i}, \pi^{-i})$ to refer to the bounds obtained by a weighted (using $\pi$) average of the bounds for individual action. When clear from context, the subscript $k$ is dropped. 

\subsection{Optimistic EBS policy}
\label{sec:optimistic_ebs_policy}
\renewcommand{\theparagraph}{\S\arabic{paragraph}}
\setcounter{secnumdepth}{4}
\paragraph{Problem formulation}
\label{sec:ebs_formulation}

Our goal is to find a game $\tilde{M}_k$ and a policy $\tilde{\pi}_k$ whose EBS value is near-optimal simultaneously for both players. In particular, if we refer to the true but unknown game by $M$ and assume that $M \in \mathcal{M}_k$ we want to find $\tilde{M}_k$ and $\tilde{\pi}_k$ such that:
\begin{align}
V_{\tilde{M}_k}(\tilde{\pi}_k) \geql  V_{M'}(\pi') \quad \forall \pi', M' \in \mathcal{M}_k \mid \Pr\left\{V_{M'}(\pi') \geq  V_{M}(\pi_{\Eg}) -(\epsilon_k, \epsilon_k)\right\} = 1\label{optimistic_objective}
\end{align}
where $\geql$ is defined in Definition \ref{definition:equilibrium} and $\epsilon_k$ a small configurable error.

Note that the condition in \eqref{optimistic_objective} is required (contrarily to single-agent games \cite{jaksch2010near}) since in general, there might not exist a game in $\mathcal{M}_k$ that achieves the highest EBS value simultaneously for both players. For example, one can construct a case where the plausible set contains two games with EBS value (resp) $(\frac{1}{2}+\epsilon, 
\frac{1}{2}+\epsilon)$ and $(\frac{1}{2},1)$ for any $0 <\epsilon < 1$ (See Table \ref{table:lower_bound_egalitarian} in Appendix \ref{proof:theo:lower_bound}). This makes the optimization problem \eqref{optimistic_objective} significantly more challenging than for single-agent games  since a small $\epsilon$ error in the rewards can lead to a large (linear) regret for one of the player. This is also the root cause for why the best possible regret becomes $\Omega(T^{2/3})$ rather than $\Omega(\sqrt{T})$ typical for single agent games. We refer the this challenge as the \challenge{} issue.

\paragraph{Solution}
\label{sec:ebs_solution}

To solve \eqref{optimistic_objective}, \textbf{a)} we set the optimistic game $\tilde{M}_k$ as the game $\MUp$ in $\mathcal{M}_k$ with the highest rewards $\rUp$ for both players. Indeed, for any policy $\pi'$ and game $M' \in \mathcal{M}_k$, one can always get a better value for both players by using $\MUp$; \textbf{b)} we compute an advantage game corresponding to $\tilde{M}_k$ by estimating an optimistic \minimax{} value for both players, a step detailed in paragraph \ref{sec:ebs_minimax}; \textbf{c)} we compute in paragraph \ref{sec:ebs_policy_computation} an EBS policy $\tilde{\pi}_{k,\Eg}$ using the advantage game; \textbf{d)} we set the policy $\tilde{\pi}_k$ to be $\tilde{\pi}_{k,\Eg}$ unless one of three conditions explained in paragraph \ref{sec:ebs_policy_execution} happens. Algorithm \ref{algo:statistical_tests} details the steps to compute $\tilde{\pi}_k$ and to correlate the policy, players play the joint-action minimizing their observed frequency of played actions compared to $\tilde{\pi}_k$ (See function $\Call{Play}$ of Algorithm \ref{algo:egalitarian_optimism} in Appendix \ref{sec:appendix_algo}).

\paragraph{Optimistic \Minimax{} Computation}
\label{sec:ebs_minimax}
Satisfying \eqref{optimistic_objective} implies we need to find a value $\SVDown$ with: \begin{equation}
\label{maximin_optimization}
\SV^i - \epsilon_k \leq \SVDown^i \leq \SV^i + \epsilon_k  \quad \forall i 
\end{equation}where $\SV^i$ is the \minimax{} value of player $i$ in the true game $M$.
To do so, we return a lower bound value for the optimistic maximin policy $\piUp_{\SV_k}^i$ of player $i$. We begin by computing in polynomial time\footnote{For example by using linear programming~\cite{dantzig1951proof, adler2013equivalence}.} the (stationary) \minimax{} policy for the game $\MUp$ with largest rewards. We then compute the (deterministic, stationary) best response policy  $\piDown_{\SVUp}^{-i}$ using the game $\MDown$ with the lowest rewards. The detailed steps are available in Algorithm \ref{algo:safe_optimism}. This results in  a lower bound on the \minimax{} value satisfying \eqref{maximin_optimization} as proven in Lemma \ref{lemma_minimax_value_bounds}.

\paragraph{Computing an EBS policy.}
\label{sec:ebs_policy_computation}
Armed with the optimistic game and the optimistic \minimax{} value, we can now easily compute the corresponding optimistic \emph{advantage game} whose rewards are denoted by $\rUp_+$. An EBS policy $\tilde{\pi}_{k, \Eg}$ is computed using this \emph{advantage game}. The key insight to do so is that the EBS involves playing a single deterministic stationary policy or combine two deterministic stationary policies (Proposition \ref{proposition:egalitarian_form}). Given that the number of actions is finite we can then just loop through each pairs of joint-actions and check which one gives the best EBS score. The score (justified in the proof of Proposition \ref{proposition:egalitarian_equal} in Appendix \ref{proof:proposition:egalitarian_equal}.) to use for any two joint-actions $a$ and $a'$ is: $\text{score}(a,a') = \min_{i \in \{1,2\}} w(a,a') \cdot \rUp^i_+(a) + (1-w(a,a')) \cdot r_+^i(a')$ with $w$ as follows:
\begin{equation}
w(a,a') = \begin{cases}
    0, & \text{if } \rUp_+^i(a) \leq \rUp^{-i}_+(a) \text{ and } \rUp^i_+(a') \leq \rUp^{-i}_+(a') \\
    1, & \text{if } \rUp_+^i(a) \geq \rUp^{-i}_+(a) \text{ and } \rUp^i_+(a') \geq \rUp^{-i}_+(a') \\
    \frac{\rUp_+^{-i}(a') - \rUp_+^{i}(a')}{\paren*{\rUp_+^i(a)-\rUp_+^i(a')} + \paren*{\rUp_+^{-i}(a') -\rUp_+^{-i}(a)}}, & \text{otherwise }
  \end{cases}
\end{equation}

And the policy $\tilde{\pi}_{k, \Eg}$ is such that
\begin{equation}
\label{eq:solution_ebs}
\tilde{\pi}_{k, \Eg}(a_{\Eg}) = w(a_{\Eg}, a'_{\Eg}); \quad \tilde{\pi}_{k, \Eg}(a'_{\Eg}) = 1-w(a_{\Eg}, a'_{\Eg}); \quad a_{\Eg}, a'_{\Eg} = \argmax_{a \in \mathcal{A}, a' \in \mathcal{A}} \text{score}(a,a')
\end{equation}


\paragraph{Policy Execution}
\label{sec:ebs_policy_execution}
We always play the optimistic EBS policy $\tilde{\pi}_{k, \Eg}$ unless one of the following three events happens:

\begin{itemize}
\item \emph{The probable error on the \minimax{} value of one player is too large}. Indeed, the error on the \minimax{} value can become too large if the weighted bound on the actions played by the \minimax{} policies is too large. In that case, we play the action causing the largest error.

\item \emph{The \challenge{} issue is probable}:
Proposition \ref{proposition:egalitarian_equal} implies that the \challenge{} issue may only happen if the player with the lowest ideal advantage value (the maximum advantage under the condition that the advantage of the other player is non-negative) is receiving it when playing an EBS policy.
This allows Algorithm \ref{algo:statistical_tests} to check for this player and plays the action corresponding to its ideal advantage as far as the other player is still receiving $\epsilon_k$-close to its EBS value (Line \ref{lst:line:ideal_player_start} to \ref{lst:line:ideal_player_end} in Algorithm \ref{algo:statistical_tests}).

\item \emph{The probable error on the EBS value of one player is too large} This only happens if we keep not playing the EBS policy due to the \challenge{} issue. In that case, the error on the EBS value used to detect the \challenge{} issue might become too large making the check for the \challenge{} issue irrelevant. In that case, we play the action of the EBS policy responsible for the largest error.
\end{itemize}
\setcounter{secnumdepth}{3}

\begin{algorithm}[t]
	\caption{Optimistic \Minimax{} Policy Computation}
	\label{algo:safe_optimism}
	\begin{algorithmic}[1]
      \Function{OptMaximin}{$\rEmp, \rUp, \rDown$.}
      \State 
        Calculate $i$'s optimistic policy: $\piUp_{\SV_k}^i = \argmax_{\pi^i} \min_{\pi^{-i}} \VUp^i(\pi^i, \pi^{-i})$\label{algo:safe_minimax_policy}
		
		\State Find the best response: $\piDown_{\SVUp_k}^{-i} = \argmin_{\pi^{-i}} \VDown^i(\piUp^i_{\SV_k}, \pi^{-i})$
		\State Get a lower bound on the \minimax{} value: $\SVDown_k^i = \min_{\pi^{-i}} \VDown^i(\piUp^i_{\SV_k}, \pi^{-i}) = \VDown^i(\piUp^i_{\SV_k}, \piDown^{-i}_{\tilde{\SV}})$\label{algo:safe_minimax_value}
		
		\State \Return $\tilde{\pi}_{\SV_k}^i$, $\piDown_{\tilde{\SV_k}}^{-i}$, $\SVDown_k^i$
        \EndFunction
	\end{algorithmic}
\end{algorithm}

\begin{algorithm}[!htbp]
	\caption{Optimistic EBS Policy Computation}
	\label{algo:statistical_tests}
	\begin{algorithmic}[1]
		\Function{OptimisticEgalitarianPolicy}{$\rEmp, \rUp, \rDown$}
		\ParState{$\piUp_{\SV_k}^i, \piDown_{\SV_k}^{-i}, \SVDown^i_k =$ \textsc{OptMaximin}($\rEmp, \rUp, \rDown$.) $\quad$ and $\quad$ $\rUp^i_+(a) = \rUp^i(a) - \SVUp^i$}
		
        \State Compute the EBS policy $\piUp_{k,\Eg}$ using \eqref{eq:solution_ebs} and $\rUp^i_+$; $\quad$ Let $\piUp_k \gets \piUp_{k,\Eg}$
%
		
		\Statex

		\ParState{($ \forall i$, from the set of actions with positive advantage $\epsilon_{t_k}$ close to the EBS value of $-i$, find the one maximizing $i$ advantage)}\label{lst:line:ideal_player_start}
		
		\begin{align*}
		\tilde{\mathcal{A}}_{i} &= \setof{a \mid \rUp_+^{i}(a) + \epsilon_{t_k} \geq \VUp_+^{i}(\piUp_{k,\Eg}) \land \rUp_+^{i}(a) \geq 0} \quad \forall i \in \setof{1,2}\\
		\aUp_{i} &= \argmax_{a \in \tilde{\mathcal{A}}_{-i}} \rUp_+^{i}(a)  \quad \forall i \in \setof{1,2}
		\end{align*}
		
		\ParState{(Look for the players $i$ whose advantage for action $\aUp_i$ is larger than the EBS value of $i$ )}
		\begin{align*}
		\tilde{\mathcal{P}} = \setof{i \in \setof{1,2} \mid \rUp_+^i(\aUp^i) > \VUp_+^{i}(\piUp_{k,\Eg})} 
		\end{align*}

		\ParState{(If there is a player whose advantage is better than the one for the EBS policy, play it)}
		
		\If{$\tilde{\mathcal{P}} \ne \varnothing$}

		\State $\tilde{p} = \argmax_{i \in \tilde{\mathcal{P}}} \rUp_+^i(\aUp^i)$ $,\quad$ $\piUp_k \gets \aUp_{\tilde{p}}$
				
		\EndIf
		
		\Statex
		
		\ParState{(If potential errors on the EBS value is too large, play the responsible action.)}\label{lst:line:minimax_error}
		\If{$2C(\piUp_{k,\Eg}) > \epsilon_{t_k}$}
		\State Let $\aUp_{k,\Eg} = \argmax_{a \in \mathcal{A} \mid C_{t_k}(a) > \epsilon_{t_k}} \piUp_{k,\Eg}(a)$ $,\quad$ $\piUp_k \gets \aUp_{k,\Eg}$

		\EndIf \label{lst:line:ideal_player_end}
		
		\Statex
		
		\ParState{(If potential errors on the \minimax{} value is too large, play the responsible action.)}
		
		\If{$2C(\piUp_{\SV_k}^i, \piDown_{\SV_k}^{-i}) > \epsilon_{t_k}$}
		
		\State Let $\aUp_{\SV_k} = \argmax_{a \in \mathcal{A} \mid C_{t_k}(a) > \epsilon_{t_k}} \piUp_{\SV_k}^i(a) \cdot \piDown_{\SV_k}^{-i}(a)$ $,\quad$ $\piUp_k \gets \aUp_{\SV_k}$ 
		
		\EndIf
		
		\State \Return $\piUp_k$
		
		\EndFunction

	\end{algorithmic}
	
  \end{algorithm}

\section{Theoretical analysis}
\label{sec:theory}

Before we present theoretical analysis for the learning algorithm, we discuss the existence and uniqueness of the EBS value, as well as the type of policies that can achieve it.
\paragraph{Properties of the EBS}

Fact \ref{fact:achievable_values} allows us to restrict our attention to stationary policies since it means that any (optimal) value achievable can be achieved by a stationary (correlated-) policy and Fact \ref{fact:egalitarian_unique} means that the egalitarian always exists and is unique providing us with a good benchmark to compare against. Fact \ref{fact:achievable_values} and \ref{fact:egalitarian_unique} are resp. justified in Appendix \ref{proof:fact:achievable_values} and \ref{proof:fact:egalitarian_unique}.
\begin{fact}[Achievable values for both players]
\label{fact:achievable_values}
\input{fact:achievable_values}
\end{fact}

\begin{fact}[Existence and Uniqueness of the EBS value for stationary policies]
\label{fact:egalitarian_unique}
\input{fact:egalitarian_unique}
\end{fact}

%

The following Proposition \ref{proposition:egalitarian_form} strengthens the observation in Fact \ref{fact:achievable_values} and establishes that a weighted combination of at most two joint-actions can achieve the EBS value. This allows for an efficient algorithm that can just loop through all possible pairs of joint-actions and check for the best one. However, given any two joint-actions one still needs to know how to combine them to get an EBS value. This question is answered by proposition \ref{proposition:egalitarian_equal}.
\begin{proposition}[On the form of an EBS policy]
\label{proposition:egalitarian_form}
\input{proposition:egalitarian_form}
\end{proposition}
\begin{proof}[Sketch]
We follow the same line of reasoning used in  \cite{poly_equilibrium_bimatrix} by showing that the EBS value lies on the outer boundary of the convex hull introduced in the proof of Fact \ref{fact:achievable_values}. This immediately implies the proposition. Details are available in Appendix \ref{proof:proposition:egalitarian_form}.
\end{proof}

\begin{proposition}[Finding an EBS policy]
\label{proposition:egalitarian_equal}
\input{proposition:egalitarian_equal}
\end{proposition}
\begin{proof}[Sketch]
Since there is an EBS policy playing only two joint-actions (by Proposition \ref{proposition:egalitarian_form}), we show how to optimally combine any two joint-actions. The proposition then follows directly. More details is available in Appendix \ref{proof:proposition:egalitarian_equal}
\end{proof}

\paragraph{Regret Analysis}
The following theorem \ref{theo:egalitarian_upper_bound} gives us a high probability upper bound on the regret in self-play against the EBS value, a result achieved without the knowledge of $T$.
\begin{theorem}[Individual Rational Regret for Algorithm \ref{algo:egalitarian_optimism} in self-play]
\label{theo:egalitarian_upper_bound}
\input{theo:egalitarian_upper_bound}
\end{theorem}
\begin{proof}[Sketch]
The structure of the proof follows that of \cite{jaksch2010near}.
The key step is to prove that the value of policy $\tilde{\pi}_k$ returned by Algorithm \ref{algo:statistical_tests} in our plausible set is $\epsilon$-close to the EBS value in the true model (optimism). In our case, we cannot always guarantee this optimism. Our proof identifies the concerned cases and show that they cannot happen too often (Lemma \ref{lemma_num_event_e} in Appendix \ref{sec:regret_egalitarian}). Then for the remaining cases, Lemma \ref{lemma_optimistm_policy} shows that we can guarantee the optimism with an error of $4\epsilon_{t_k}$. 
The step-by-step detailed proof is available in Appendix \ref{sec:regret_egalitarian}.
\end{proof}

By definition of EBS, Theorem \ref{theo:egalitarian_upper_bound} also applies to the safety regret. However in Theorem \ref{theo:safe_regret}, we show that the optimistic \minimax{} policy enjoys near-optimal safety regret of $\BigO(\sqrt{T})$.
\begin{theorem}[Safety Regret of policy $\tilde{\pi}_{\SV_k}^i$ in Algorithm \ref{algo:safe_optimism}]
\label{theo:safe_regret}
\input{theo:safe_regret}
\end{theorem}
\begin{proof}[Sketch]
The proof works similarly to that of Theorem \ref{theo:egalitarian_upper_bound} by observing that here we can always guarantee optimism. A more detailed proof is available in Appendix \ref{proof:theo:safe_regret}.
\end{proof}

\paragraph{Lower bounds for the individual rational regret}

Here we establish a lower bound of $\Omega\paren*{T^{2/3}}$ for any algorithm trying to learn the EBS value. This shows that our upper bound is optimal up to logarithm-factors. The key idea in proving this lower bound is the example illustrated by Table \ref{table:lower_bound_egalitarian}. In that example, the rewards of the first player are all $\frac{1}{2}$ and the second player has an ideal value of $1$. However, 50\% of the times a player cannot realize its ideal value due to an $\epsilon$-increase in a single joint-action for both players. The main intuition behind the proof of the lower bound is that any algorithm that wants to minimize regret can only try two things \textbf{(a)} detect whether there exists a joint-action with an $\epsilon$ or if all rewards of the first player are equal. \textbf{(b)} always ensure the ideal value of the second player. To achieve \textbf{(a)} any algorithm needs to play all joint-actions for $\frac{1}{\epsilon^2}$ times. Picking $\epsilon = T^{-1/3}$ ensures the desired lower bound. The same $\epsilon$ would also ensure the same lower bound for an algorithm targeting only \textbf{(b)}. Appendix \ref{proof:theo:lower_bound} formally proves this lower bound.

%
%
%
%
%

\begin{theorem}[Lower bounds]
\label{theo:lower_bound}
\input{theo:lower_bound}
\end{theorem}






\section{Conclusion and Future Directions}
\label{sec:conclusion}

In this paper, we illustrated a situation in which typical solutions for self-play in repeated games, such as single-stage equilibrium or sum of rewards, are not appropriate. We propose the usage of an egalitarian bargaining solution (EBS) which guarantees each player to receive no less than their \minimax{} value. We analyze the properties of EBS for repeated games with stochastic rewards and derive an algorithm that achieves a near-optimal finite-time regret of $\BigO(T^{2/3})$ with high probability. We are able to conclude that the proposed algorithm is near-optimal, since we prove a matching lower bound up to logarithmic-factor. Although our results imply a $\BigO(T^{2/3})$ safety regret (i.e. compared to the \minimax{} value), we also show that a component of our algorithm guarantees the near-optimal $\BigO(\sqrt{T})$ safety regret against arbitrary opponents.

Our work illustrates an interesting property of the EBS which is: it can be achieved with sub-linear regret by two individually rational agents who are uncertain about their utility. We wonder if other solutions to the Bargaining Problem such as the Nash Bargaining Solution or the Kalai–Smorodinsky Solution also admit the same property. Since the EBS is an equilibrium, another intriguing question is whether one can design an algorithm that converges naturally to the EBS solution against some well-defined class of opponents.

Finally, a natural and interesting future direction for our work is its extension to stateful games such as Markov games.




\bibliographystyle{apalike}
\bibliography{bibliography}  

\appendix

\section{Notations and terminology}
We will use action to mean joint-actions unless otherwise specified. We will denote the players as $i$ and $-i$. This is to be understood as follows: if there are two players $\setof{1,2}$, when $i=1$, then $-i=2$ and when $i=2$, $-i=1$. The true but unknown game will be denoted as $M$ whereas the plausible set of games we consider at \epoch{} $k$ will be denoted by $\mathcal{M}_k$. An EBS policy in the true game $M$ will be denoted by $\piEg[]$ and its value by $\VEg[]$. If for the EBS value in $M$, the player with the lowest ideal advantage value is receiving it, we will denote this player by $p^-$ while the other player will be $p^+$. The EBS policy in this situation will be denoted as $a^*$ (it is guaranteed to be a single joint-action).

$\bar{r}$ will be used to denote empirical mean rewards and in general $\;\bar{}\;$ is used to mean a value computed using empirical $\bar{r}$.
$\tilde{r}$ will be used to mean the rewards from the upper limit game in our plausible set, while $\hat{r}$ will be used to mean the rewards from the lower limit game in our plausible set. Also, in general $\;\tilde{}\;$ while be used to mean a value computed using $\tilde{r}$ and $\;\hat{}\;$ to mean a value computed using $\hat{r}$.

$k$ will be used to denote the current \epoch{}.
$N_k(a)$ the number of \rounds{} action $a$ has been played in \epoch{} $k$ --- $N_k$ the number of \rounds{} \epoch{} $k$ has lasted --- $t_k$ the number of rounds played up to \epoch{} $k$ --- $N_{t_k}(a)$ the number of \rounds{} action $a$ has been played up to round $t_k$ --- $\bar{r}_t^i(a)$ the empirical average rewards of player $i$ for action $a$ at round $t$.
$m$ will be used to denote the total number of \epochs{} up to \round{} $T$.

\section{Proof of Theorem \ref{theo:egalitarian_upper_bound}}
\begin{theorem*}[\ref{theo:egalitarian_upper_bound}]
	\input{theo:egalitarian_upper_bound}
\end{theorem*}

\subsection{Regret analysis for the egalitarian algorithm in self-play}
\label{sec:regret_egalitarian}

The proof is similar to that of UCRL2 \cite{jaksch2010near} and KL-UCRL \cite{filippi2010optimism}. As the algorithm is divided into \epochs{}, we first show that the regret bound within an \epoch{} is sub-linear. We then combine those per-\epoch{} regret terms to get a regret for the whole horizon simultaneously. Both of these regrets are computed with the assumption that the true game $M$ is within our plausible set. We then conclude by showing that this is indeed true with high probability. Let's first start by decomposing the regret.

\paragraph{Regret decomposition}

Here we decompose the regret in each round $k$. We start by defining the following event $E$, 
\begin{align}
\label{regret_event}
E &= E_1 \lor E_2 \lor E_3 \lor E_4,
\end{align}

\begin{align}
E_1&: \paren*{\tilde{\pi}_k = \tilde{a}_{\SV_k}}\label{regret_event_minimax}\\
E_2&: \paren*{\tilde{\pi}_k = \tilde{a}_{k,\Eg}}\label{regret_event_egalitarian}\\
E_3&: \paren*{a_* \notin \tilde{\mathcal{A}}_{p^-} \land 2C_{t_k}(\tilde{\pi}_{k,\Eg}) \leq \epsilon_{t_k} \mid  \pi_{Eg} = a_{*} }\label{regret_event_correct_min}\\
E_4 &: \paren*{a_{*} \in \tilde{\mathcal{A}}_{p^-} \land \tilde{\pi}_k = \tilde{a}_{p^-} \land p^+ \in \tilde{\mathcal{P}} \mid \pi_{Eg} = a_{*}, V_+^{p^+}(a_{*}) > V_+^{p^-}(a_{*}) +  2\epsilon_{t_k}}\label{regret_event_incorrect_min}
\end{align}

We have:

\begin{align}
\regret_T^i &= \sum_{t=1}^{T} \VEg[]^i - r_t^i\\
&= \sum_{t=1}^{T}\Id_{E=1} \paren*{\VEg[]^i - r_t^i} + \sum_{t=1}^{T}\Id_{E=0}\paren*{\VEg[]^i - r_t^i}\\
&\leq \sum_{t=1}^{T}\Id_{E=1} + \sum_{t=1}^{T}\Id_{E=0}\paren*{\VEg[]^i - r_t^i}\label{eq:reg_decomposition_event}
\end{align}

In the following, we will use Hoeffding's inequality to bound the last term of Equation \ref{eq:reg_decomposition_event}, similarly to Section 4.1 in \cite{jaksch2010near}. In particular, with probability at least $1-\delta'$:
\begin{align}
\regret_T^i &\leq \sum_{t=1}^{T}\Id_{E=1} + \sum_{k=1}^{m} \giventhat*{\Delta_k}{E=0} + \sqrt{T \ln(1/\delta')/2}\label{regret_decomposed}
\end{align}

where $\Delta_{k}$ is the regret per-\epoch{} defined by \begin{align}
\Delta_k = \sum_{a \in \mathcal{A}} N_{k}(a) \paren*{\VEg[]^i - \E r^i(a)}
\end{align}

\paragraph{Regret when the event E defined by \eqref{regret_event} is False and the true Model is in our plausible set}

We will now simplify the notation by using $\Delta_{k, \neg E}$ to mean that the expression is condition on $E$ being \emph{False}. We can thus bound $\Delta_{k,\neg E}$:

\begin{align}
\Delta_{k,\neg E} &\leq \sum_{a \in \mathcal{A}} N_{k}(a) \paren*{\VUp_k^i - \E r^i(a) + 4\epsilon_{t_k}} \label{eq:Delta_optimism}\\ 
&= \sum_{a \in \mathcal{A}} N_{k}(a) \paren*{\VUp_k^i - \rUp^i(a)} + \sum_{a \in \mathcal{A}} N_{k}(a) \paren*{\rUp^i(a) - \E r^i(a)} + 4\sum_{a \in \mathcal{A}} N_{k}(a) \epsilon_{t_k}\notag\\
&= \sum_{a \in \mathcal{A}} N_{k}(a) \paren*{\VUp_k^i - \rUp^i(a)}  + 4\sum_{a \in \mathcal{A}} N_{k}(a) \epsilon_{t_k}  + \sum_{a \in \mathcal{A}} N_{k}(a) \paren*{\rEmp^i(a) - \E r^i(a) + \frac{C_r(t_k)}{\sqrt{N_{t_k}(a)}} } \notag\\
&\leq \sum_{a \in \mathcal{A}} N_{k}(a) \paren*{\VUp_k^i - \rUp^i(a)}  + 4\sum_{a \in \mathcal{A}} N_{k}(a) \epsilon_{t_k}   + \sum_{a \in \mathcal{A}}  \frac{2C_r(t_k)N_{k}(a)}{\sqrt{N_{t_k}(a)} } \label{eq:Delta_good_region}\\
&\leq \size{\mathcal{A}} + 4\sum_{a \in \mathcal{A}} N_{k}(a) \epsilon_{t_k} + \sum_{a \in \mathcal{A}}  \frac{2C_r(t_k)N_{k}(a)}{\sqrt{N_{t_k}(a)} }\label{eq:Delta_egalitarian}
\end{align}

where Equation \eqref{eq:Delta_optimism} comes from the fact that when $E=0$, $\VUp_k^i \geq \VBR[]^i - 4\epsilon_{t_k}$ (See Lemma \ref{lemma_optimistm_policy}). Equation \eqref{eq:Delta_good_region} comes from the fact that we assume $M \in \mathcal{M}_k$ meaning $\abs{\rEmp^i(a) - \E r^i(a) } \leq \frac{C_r(t_k)}{\sqrt{N_{t_k}(a)}}$. Equation \eqref{eq:Delta_egalitarian} comes from the fact the egalitarian solution involves playing one joint-action with probability $w_k \in [0,1]$ and another joint-action with probability $1-w_k$; since it is always possible to bound $w_k$ as $\frac{n}{N_k} \leq w_k \leq \frac{n+1}{N_k}$ with $n \in \mathbb{N}$ a non-negative integer, and by construction the players play as close as possible to $w_k$, then the error is bounded by $\frac{1}{N_k} \leq \frac{1}{N_k(a)}$.

We are now ready to sum up the per-\epoch{} regret over all \epochs{} for which the event $E$ is false. We have:

\begin{align}
\sum_{k=1}^{m}\Delta_{k} &\leq \sum_{k=1}^{m} \paren*{\size{\mathcal{A}} + 4\sum_{a \in \mathcal{A}} N_{k}(a) \epsilon_{t_k} + \sum_{a \in \mathcal{A}}  \frac{2C_r(t_k)N_{k}(a)}{\sqrt{N_{t_k}(a)} }}\\
&= m\size{\mathcal{A}} + 4\sum_{k=1}^{m}\epsilon_{t_k}N_k + \sum_{k=1}^{m}\sum_{a \in \mathcal{A}}\frac{2C_r(t_k)N_{k}(a)}{\sqrt{N_{t_k}(a)} }
\end{align}

Now assuming $\epsilon_{t_k} = C_e \cdot\paren*{\frac{\numact \ln t_k}{t_k}}^{1/3}$, we have:
\begin{align}
\sum_{k=1}^{m}\epsilon_{t_k}N_k&= \sum_{k=1}^{m} \frac{\numact^{1/3}C_e \sqrt[3]{\ln t_k}}{t_{k}^{1/3}}N_k\\
&\leq T^{1/6}\numact^{1/3}C_e\sqrt[3]{\ln T} \sum_{k=1}^{m} \frac{N_k}{\sqrt{t_k}}
\end{align}

Using Appendix C.3 in \cite{jaksch2010near}, we can conclude that 

\[\sum_{k=1}^{m}\epsilon_{t_k}N_k \leq T^{2/3}\numact^{1/3}C_e\sqrt[3]{\ln T}\]

Similarly \cite{jaksch2010near} Equation (20) shows that:

\[\sum_{k=1}^{m}\sum_{a \in \mathcal{A}}\frac{N_{k}(a)}{\sqrt{N_{t_k}(a)} } \leq (\sqrt{2}+1)\sqrt{\size{\mathcal{A}}T}\]

Furthermore \cite{jaksch2010near} shows that:

\[ m \leq \size{\mathcal{A}} \log_2\paren*{\frac{8T}{\size{\mathcal{A}}}}\]

Combining all the above results lead to
\begin{align}
\sum_{k=1}^{m}\Delta_{k} &\leq \size{\mathcal{A}}^2 \log_2\paren*{\frac{8T}{\size{\mathcal{A}}}} + 4T^{2/3}\numact^{1/3}C_e\sqrt[3]{\ln T}+ 2\max_{k}C_r(t_k)(\sqrt{2}+1)\sqrt{\size{\mathcal{A}} T}\label{regret_sum_epoch}
\end{align}

\paragraph{Combining with probability of failure}
Combining \eqref{regret_sum_epoch} with Lemma \ref{lemma_num_event_e} bounding the number of times event $E$ is true, together with Proposition \ref{proposition:failure_prob} justifying the high probability from $t \geq \max\setof{3,\paren*{T \numact}^{1/4}}$, noticing that up to $t = \max\setof{3,\paren*{T \numact}^{1/4}}$ our maximum regret is $\max\setof{6,2\paren*{T \numact}^{1/4}}$, picking $\delta'$ in \eqref{regret_decomposed} as $\delta' = \frac{\delta^4}{T} \cdot \paren*{\frac{99.375-\pi^4}{45}}$ and using $\max_{k} C_r(t_k) = C_r(t_m)$ leads to:

\begin{align*}
	\regret_T &\leq \size{\mathcal{A}}^2 \log_2\paren*{\frac{8T}{\size{\mathcal{A}}}} + 4T^{2/3}\numact^{1/3}C_e\sqrt[3]{\ln T}+ 2\sqrt{2\ln \frac{\numact T \log_2 \frac{8T}{\numact}}{\delta}}(\sqrt{2}+1)\sqrt{\size{\mathcal{A}} T}\\
	&\quad  + \frac{16A^{1/3}T^{2/3}\ln^{1/3}T}{C_e^2} + 2\sqrt{\sqrt{T A}} + 6 + \sqrt{\frac{T}{2} \cdot \ln\paren*{\frac{T}{\delta^4}\cdot \frac{45}{99.375-\pi^4} }}
\end{align*}

with a failure probability of $\frac{\delta^4}{3T}$

Picking $C_e = 2$ leads to the statement of Theorem \ref{theo:egalitarian_upper_bound}.

\begin{proposition}[Probability of Failure]
	\label{proposition:failure_prob}
	If Algorithm \ref{algo:egalitarian_optimism} is run with the plausible set constructed with $\delta_{t_k}  = \frac{\delta}{ k \cdot t_k}$, then the probability of failure from \round{} $t=\max\setof{3,\paren*{T \numact}^{1/4}}$ to \round{} $T$ with $T \geq 3$ is upper bounded as:
	
	\[\Prob\curly*{\bigcup_{k \mid t_k \geq \sqrt{T\numact}} M \notin \mathcal{M}_{k}} \leq \frac{\delta^4}{T} \cdot \paren*{\frac{\pi^4-84.375}{45}}\]
\end{proposition}

\begin{proof}
	Using Lemma \ref{lemma_high_prob_set}, we know that at any \epoch{} $k$, our plausible set contains the true model with probability at least $1-2\size{\mathcal{A}}\delta_{t_k}^4$, meaning that the probability of failure is $2\size{\mathcal{A}}\delta_{t_k}^4$.
	We want to combine these failure probabilities for all possible \epochs{} (after \round{} $\max\setof{3,\paren*{T \numact}^{1/4}}$). For that, consider  $\delta_{t_k} = \frac{\delta}{\cdot k \cdot t_k}$. Let $k_t$ the smallest \round{} such that $t_{k_t} \geq \max\setof{3,\paren*{T \numact}^{1/4}}$. Using a union bound (Fact \ref{union_bound}) we have that the probability $\Prob\curly*{\bigcup_{k \mid t_k \geq \max\setof{3,\paren*{T \numact}^{1/4}}} M \notin \mathcal{M}_{k}}$ up to \round{} $T$ is:
	
	\begin{align}
	\Prob\curly*{\bigcup_{k \mid t_k \geq \max\setof{3,\paren*{T \numact}^{1/4}}} M \notin \mathcal{M}_{k}} &\leq \sum_{k=k_t}^{m} 2\size{\mathcal{A}}\delta_{t_k}^4\\
	&=\sum_{k=k_t}^{m} \frac{2\numact\delta^4}{k^4\cdot t_k^4}\\
	&\leq\frac{2\delta^4}{T}\sum_{k=k_t}^{m} \frac{1}{k^4}\\
	&=\frac{2\delta^4}{T}\paren*{\sum_{k=1}^{m} \frac{1}{k^4}-\sum_{k=1}^{k_t-1} \frac{1}{k^4}}\\
	&\leq  \frac{2\delta^4}{T}  \paren*{\frac{\pi^4}{90}  -1 -\frac{1}{16}}\label{failure_using_epoch_bound}
	\end{align}
	
	\eqref{failure_using_epoch_bound} comes from $\sum_{k=1}^{\infty} \frac{1}{k^4} = \frac{\pi^4}{90}$ and $k_t\geq3$ when $t_{k_t}\geq3$ due to the doubling trick used in Algorithm \ref{algo:egalitarian_optimism}.
	
\end{proof}

The following fact comes directly by construction of the Algorithm.
\begin{fact}[Error on the \minimax{} policy]
\label{fact_minimax_error}
For any player $i$ and \epoch{} $k$ for which the true model $M$ is within our \emph{plausible set} $\mathcal{M}_k$ and Event $E_1$ defined by \eqref{regret_event_minimax} is False, then:
\[ 2C_{t_k}(\tilde{\pi}^i_{\SV_k}, \hat{\pi}^{-i}_{\SV_k}) \leq \epsilon_{t_k}\]
\end{fact}

\begin{lemma}[Pessimism and Optimism of the \minimax{} value]
\label{lemma_minimax_value_bounds}
For any player $i$ and \epoch{} $k$ for which the true model $M$ is within our \emph{plausible set} $\mathcal{M}_k$, the \minimax{} value computed satisfies:
\[ \SV^i - 2C_{t_k}(\tilde{\pi}^i_{\SV_k}, \hat{\pi}^{-i}_{\SV_k}) \leq \SVDown^i_k \leq \SV^i \]
\end{lemma}

\begin{proof}
\textbf{Pessimism of the \minimax{} value.}
Let $\hat{V}$ denote the lower bound on the value through the reward estimate $\bar{r} - C$ and $\tilde{V}$ the corresponding upper bound through $\bar{r} + C$. By definition,
\begin{align}
	\tilde{\pi}_{\SV_k}^i &= \argmax_{\pi^i} \min_{\pi^{-i}} \tilde{V}^i(\pi^i, \pi^{-i})\\
	\hat{\pi}_{\SV_k}^{-i} &= \argmin_{\pi^{-i}} \hat{V}^i(\tilde{\pi}^i_{\SV_k}, \pi^{-i})\\
	\SVDown^i_k &= \min_{\pi^{-i}} \hat{V}^i(\tilde{\pi}^i_{\SV_k}, \pi^{-i})\\
	&= \hat{V}^i(\tilde{\pi}^i_{\SV_k}, \hat{\pi}^{-i}_{\SV_k})
\end{align}

As a result we have:
\begin{align}
	\SV^i &= \max_{\pi^i} \min_{\pi^{-i}} V^i(\pi^i, \pi^{-i})\\
	&\geq \min_{\pi^{-i}} V^i(\tilde{\pi}^i_{\SV_k}, \pi^{-i})\\
	&\geq \min_{\pi^{-i}} \hat{V}^i(\tilde{\pi}^i_{\SV_k}, \pi^{-i})\\
	&= \SVDown^i_k
\end{align}

\paragraph{Optimism of the \minimax{} value}
We have:

\begin{align}
	\SVDown^i_k &= \hat{V}^i(\tilde{\pi}^i_{\SV_k}, \hat{\pi}^{-i}_{\SV_k})\\
	&= \tilde{V}^i(\tilde{\pi}^i_{\SV_k}, \hat{\pi}^{-i}_{\SV_k}) - 2C_{t_k}(\tilde{\pi}^i_{\SV_k}, \hat{\pi}^{-i}_{\SV_k})\\
	&\geq \min_{\pi^{-i}}\tilde{V}^i(\tilde{\pi}^i_{\SV_k}, \pi^{-i}) - 2C_{t_k}(\tilde{\pi}^i_{\SV_k}, \hat{\pi}^{-i}_{\SV_k})\\
	&=\tilde{V}^i(\tilde{\pi}^i_{\SV_k}, \tilde{\pi}^{-i}_{\SV_k}) - 2C_{t_k}(\tilde{\pi}^i_{\SV_k}, \hat{\pi}^{-i}_{\SV_k})\\
	&= \max_{\pi^i}\tilde{V}^i(\pi^i, \tilde{\pi}^{-i}_{\SV_k}) - 2C_{t_k}(\tilde{\pi}^i_{\SV_k}, \hat{\pi}^{-i}_{\SV_k})\\
	&\geq \max_{\pi^i}V^i(\pi^i, \tilde{\pi}^{-i}_{\SV_k}) - 2C_{t_k}(\tilde{\pi}^i_{\SV_k}, \hat{\pi}^{-i}_{\SV_k})\\
	&\geq \max_{\pi^i} \min_{\pi^{-i}}V^i(\pi^i, \pi^{-i})- 2C_{t_k}(\tilde{\pi}^i_{\SV_k}, \hat{\pi}^{-i}_{\SV_k})\\
	&= \SV^i - 2C_{t_k}(\tilde{\pi}^i_{\SV_k}, \hat{\pi}^{-i}_{\SV_k})\\
\end{align}
\end{proof}

\begin{lemma}[Optimism of the advantage game]
\label{lemma_optimism_advantage}
This lemma prove that the advantage value for any policy $\pi$ in our optimistic model is greater than in the true model. For any policy $\pi$, player $i$ and \epoch{} $k$ for which the true model $M$ is within our \emph{plausible set} $\mathcal{M}_k$:
\[\tilde{V}_+^i(\pi) \geq V_+^i(\pi)\]
\end{lemma}

\begin{proof}
First, we prove that the advantage value for any policy $\pi$ in our optimistic model is greater than in the true model. We have:

\begin{align}
	\tilde{V}_+^i(\pi) &= \tilde{V}^i(\pi) - \SVDown^i_k\\
	&\geq V^i(\pi) - \SVDown^i_k \\
	&\geq V^i(\pi) - \SV^i \label{opt_minimax_used}\\
	&= V_+^i(\pi)
\end{align}

where \eqref{opt_minimax_used} comes from Lemma \ref{lemma_minimax_value_bounds}.
\end{proof}

\begin{lemma}[Optimism of the Policy computation]
\label{lemma_optimistm_policy}
For any \epoch{} $k$ for which the true model $M$ is in our plausible set $\mathcal{M}_k$ and Event $E$ defined by \eqref{regret_event} is False, then for any player $i$, we have:
\[\VUp_k^i\geq \VEg[]^i -4\epsilon_{t_k}\]
\end{lemma}
\begin{proof}
Immediate by combining Lemma \eqref{lemma_optimism_same_advantage} \eqref{lemma_optimism_max_advantage}, \eqref{lemma_optimism_not_max_advantage}
\end{proof}

\begin{lemma}[Number of times event $E$ defined by \eqref{regret_event} is True]
\label{lemma_num_event_e}
After any number $T$ of \rounds{} for which the true model $M$ was in our plausible set $\mathcal{M}_k$ at every \epochs{} $k$ up to $T$, the number $N_T(E = \text{True})$ of \rounds{}  for which Event $E$ defined by \eqref{regret_event} is True satisfy:

\[N_T(E = \text{True}) \leq \frac{16A^{1/3}T^{2/3}\ln^{1/3}T}{C_e^2}\]
\end{lemma}
\begin{proof}
The proof is an immediate consequence of combining Lemmas \ref{lemma_num_event_e1},\ref{lemma_num_event_e2},\ref{lemma_num_event_e3}, \ref{lemma_num_event_e4}.
\end{proof}

\begin{lemma}[High probability of the plausible set]
\label{lemma_high_prob_set}

At any epoch $k$, with probability at least $1-2\size{\mathcal{A}}\delta^4$, the true model $M$ is in our plausible set constructed using \eqref{plausible_set}.
\end{lemma}

\begin{proof}
This is a direct application of Hoeffding's bound (Lemma \ref{hoeffding}) on each action individually and then using a union bound (Fact \ref{union_bound}) over all actions.
\end{proof}

\begin{lemma}[Number of times an action from the \minimax{} policy is played]
\label{lemma_num_event_e1}

After any number $T$ of \rounds{} for which the true model $M$ was in our plausible set $\mathcal{M}_k$ at every \epochs{} $k$ up to $T$, the number $N_T(E_1 = \text{True})$ of \rounds{} for which event $E_1$ defined by \eqref{regret_event_minimax} is true satisfy:

\[N_T(E_1 = \text{True}) \leq \frac{16A^{1/3}T^{2/3}\ln^{1/3}T}{C_e^2}\]

\end{lemma}

\begin{proof}
By definition, the event $E_1$ is true when $\tilde{\pi}_k = \tilde{a}_{\SV_k}$. We have:
\begin{align}
\tilde{\pi}_k = \tilde{a}_{\SV_k} &\implies
2C_{t_k}(\tilde{a}_{\SV_k})> \epsilon_{t_k}\\
&\implies 2\sqrt{\frac{2\ln t_k}{N_{t_k}(\tilde{a}_{\SV_k})}}> \epsilon_{t_k}\\
&\implies N_{t_k}(\tilde{a}_{\SV_k}) < \frac{8 \ln t_k}{\epsilon_{t_k}^2}
\end{align}
Summing over all actions, using the fact that during \epoch{} $k$ no action is played more than twice the number of times it was played at the beginning of \epoch{} $k$ and replacing $\epsilon_{t_k}$ by its value, leads to the statement of the lemma.
\end{proof}

\begin{lemma}[Number of times event defined by \eqref{regret_event_egalitarian} is True]
\label{lemma_num_event_e2}

After any number $T$ of \rounds{} for which the true model $M$ was in our plausible set $\mathcal{M}_k$ at every \epochs{} $k$ up to $T$, the number $N_T(E_2 = \text{True})$ of \rounds{} for which event $E_2$ defined by \eqref{regret_event_egalitarian} is true satisfy:

\[N_T(E_2 = \text{True}) \leq \frac{16A^{1/3}T^{2/3}\ln^{1/3}T}{C_e^2}\]

\end{lemma}

\begin{proof}
By definition, the event is true when $\tilde{\pi}_k = \tilde{a}_{k,\Eg}$. We have:
\begin{align}
\tilde{\pi}_k = \tilde{a}_{k, \Eg} &\implies
2C_{t_k}(\tilde{a}_{k, \Eg})> \epsilon_{t_k}
\end{align}

And the remainder of the proof follows the proof of Lemma \ref{lemma_num_event_e1}
\end{proof}

\begin{lemma}[Number of times Event defined by \eqref{regret_event_correct_min} is True]
\label{lemma_num_event_e3}
After any number $T$ of \rounds{} for which the true model $M$ was in our plausible set $\mathcal{M}_k$ at every \epochs{} $k$ up to $T$, Event $E_3$ defined by \eqref{regret_event_correct_min} is always False.
\end{lemma}

\begin{proof}

\begin{align}
E_3 = \text{True} &\implies \tilde{r}_+^{p^-}(a_{*})  + \epsilon_{t_k} < \tilde{V}_+^{p^-}(\tilde{\pi}_{k,\Eg}) \label{E3_step_definition}\\
&\implies \bar{r}_+^{p^-}(a_{*}) +  \sqrt{\frac{2 \ln t_k}{N_{t_k}(a_{*})}} + \epsilon_{t_k} < \bar{V}_+^{p^-}(\tilde{\pi}_{k,\Eg}) + C_{t_k}(\tilde{\pi}_{k,\Eg})\\
&\implies \E r_+^{p^-}(a_{*}) + \epsilon_{t_k} + \SV^{p^-}-\SVDown^{p^-}_k< V_+^{p^-}(\tilde{\pi}_{k,\Eg}) +  2C_{t_k}(\tilde{\pi}_{k,\Eg})+ \SV^{p^-} - \SVDown^{p^-}_k\label{E3_step_true_model}\\
&\implies \E r_+^{p^-}(a_{*}) + \epsilon_{t_k} < \E r_+^{p^-}(a_{*}) + \epsilon_{t_k} \quad(\text{which is always False})\label{E3_step_assumptions}
\end{align}

\eqref{E3_step_definition} comes from $a_{*} \notin \tilde{\mathcal{A}}_{p^-}$  (by definition)

\eqref{E3_step_true_model} comes from (since the true model is in our plausible set)\[\E r_+^{p^-}(a)-C_{t_k}(a) + \SV^{p^-} - \SVDown^{p^-}_k \leq \bar{r}_+^{p^-}(a) \leq \E r_+^{p^-}(a)+C_{t_k}(a) + \SV^{p^-} - \SVDown^{p^-}_k\; \forall a \in \mathcal{A}\] 

\eqref{E3_step_assumptions} comes from $\E r_+^{p^-}(a_{\Eg}) \geq V_+^{p^-}(\tilde{\pi}_{k,\Eg})$ (since by assumption $\pi_{Eg} = a_{*}$) and $2C(\tilde{\pi}_{k,\Eg}) \leq \epsilon_{t_k}$ (by assumption).
\end{proof}

\begin{lemma}[Number of times Event defined by \eqref{regret_event_incorrect_min} is True]
\label{lemma_num_event_e4}
After any number $T$ of \rounds{} for which the true model $M$ was in our plausible set $\mathcal{M}_k$ at every \epochs{} $k$ up to $T$, the number $N_T(E_4 = \text{True})$ of \rounds{}  for which Event $E_4$ defined by \eqref{regret_event_incorrect_min} is True satisfy:

\[N_T(E_4 = \text{True}) \leq \frac{16A^{1/3}T^{2/3}\ln^{1/3}T}{C_e^2}\]
\end{lemma}

\begin{proof}
The condition $E_4$ is: $a_{*} \in \tilde{\mathcal{A}}_{p^-} \land \tilde{\pi}_k = \tilde{a}_{p^-} \land p^+ \in \tilde{\mathcal{P}} \mid \pi_{Eg} = a_{*}, V_+^{p^+}(a_{*}) > V_+^{p^-}(a_{*}) +  2\epsilon_{t_k}$

\begin{align}
E_4\;\text{is True} &\implies \tilde{r}^{p^-}_+(\tilde{a}_{p^-}) \geq \tilde{r}^{p^+}_+(\tilde{a}_{p^+})\label{E4:step:definition}\\
&\implies\tilde{r}^{p^-}_+(\tilde{a}_{p^-})\geq \tilde{r}_+^{p^+}(a_{*})\label{E4:step:definition2}\\
&\implies \tilde{r}^{p^-}_+(\tilde{a}_{p^-})\geq \E r_+^{p^+}(a_{*})\label{E4:step:lemma}\\
&\implies \tilde{r}^{p^-}_+(\tilde{a}_{p^-}) > \E r_+^{p^-}(a_{*}) +  2\epsilon_{t_k}\label{E4:step:definition3}\\
&\implies \E r^{p^-}_+(\tilde{a}_{p^-}) +2C_{t_k}(\tilde{a}_{p^-}) + \SV^{p^-} - \SVDown^{p^-}_k > \E r_+^{p^-}(a_{*}) +  2\epsilon_{t_k}\label{E4:step:definition4}\\
&\implies \E r^{p^-}_+(\tilde{a}_{p^-}) +2C_{t_k}(\tilde{a}_{p^-}) + \epsilon_{t_k}> \E r_+^{p^-}(a_{*}) +  2\epsilon_{t_k}\label{E4:step:definition5}\\
&\implies  2C_{t_k}(\tilde{a}_{p^-}) + \epsilon_{t_k}> 2\epsilon_{t_k}\label{E4:step:definition6}\\
&\implies N_{t_k}(\tilde{a}_{p^-}) < \frac{8 \ln t_k}{\epsilon_{t_k}^2}  \quad \text{(with}\; \tilde{\pi}_k = \tilde{a}_{p^-}\text{)}
\end{align}

where \eqref{E4:step:definition} comes (by definition of event $E_4$) $\tilde{\pi}_k = \tilde{a}_{p^-} \land p^+ \in \tilde{\mathcal{P}} $

\eqref{E4:step:definition2} comes from (by definition of event $E_4$) $a_* \in \tilde{\mathcal{A}}_{p^-}$ and the definition of $\tilde{a}_{p^+}$

\eqref{E4:step:lemma} comes from using Lemma \ref{lemma_optimism_advantage} since the true model is in our plausible set.

\eqref{E4:step:definition3} comes from the definition of event $E_4$ that $V_+^{p^+}(a_{*}) > V_+^{p^-}(a_{*}) +  2\epsilon_{t_k}$

\eqref{E4:step:definition4} comes from (since the true model is in our plausible set)\[\bar{r}_+^{p^-}(a) \leq \E r_+^{p^-}(a)+C_{t_k}(a) + \SV^{p^-} - \SVDown^{p^-}_k\; \forall a \in \mathcal{A}\]

\eqref{E4:step:definition5} comes from Lemma \ref{lemma_minimax_value_bounds} (since the true model is in our plausible set) and Fact \ref{fact_minimax_error} (since Event $E_1$ is False)

\eqref{E4:step:definition6} comes from $\E r_+^{p^-}(a_{*}) \geq \E r^{p^-}_+(\tilde{a}_{p^-})$ since $\pi_{\Eg} = a_{*}$

Summing over all actions, using the fact that during \epoch{} $k$ no action is played more than twice the number of times it was played at the beginning of \epoch{} $k$ and replacing $\epsilon_{t_k}$ by its value, leads to the statement of the lemma.
\end{proof}

\begin{lemma}[Optimism of the Policy computation when egalitarian advantage is identical]
\label{lemma_optimism_same_advantage}
For any \epoch{} $k$ for which the true model $M$ is in our plausible set $\mathcal{M}_k$, the two players have the same egalitarian advantage value in $M$ and Event $E$ defined by \eqref{regret_event} is False, then for any player $i$, we have:

\[\VUp_k^i \geq \VEg[]^i -2\epsilon_{t_k}\]

\begin{proof}
We have:

\begin{align}
	\VUp_k^i &= \tilde{V}^i(\tilde{\pi_k})\\
	&\geq \tilde{V}^i(\tilde{\pi}_{k,\Eg}) -\epsilon_{t_k}\\
	&= \tilde{V}_+^i(\tilde{\pi}_{k,\Eg}) + \SVUp_k^i -\epsilon_{t_k}\\
	&\geq \min_j \tilde{V}_+^j(\pi_{\Eg}) + \SVUp_k^i -\epsilon_{t_k}\label{egal:step_min}\\
	&\geq \min_j V_+^j(\pi_{\Eg}) + \SVUp_k^i -\epsilon_{t_k}\\
	&= V_+^i(\pi_{\Eg})+  \SVUp_k^i -\epsilon_{t_k}\label{egal:step_identical}\\
	&= V_+^i(\pi_{\Eg})+ \SV^i - \SV^i +  \SVUp_k^i -\epsilon_{t_k}\\
	&= V^i(\pi_{\Eg}) + \SVUp_k^i - \SV^i-\epsilon_{t_k}\\
	&\geq V^i(\pi_{\Eg}) -2C_{t_k}(\tilde{\pi}^i_{\SV_k}, \hat{\pi}^{-i}_{\SV_k})-\epsilon_{t_k}\label{egal:step_minimax_error}\\
	&= \VEg[]^i -2\epsilon_{t_k}\label{egal:step_minimax_error_fact}
\end{align}

\eqref{egal:step_min} comes by definition of $\tilde{\pi}_{k,\Eg}$ (i.e the policy maximizing the minimum advantage).

\eqref{egal:step_identical} comes by assumption that in the true model both players have the same egalitarian advantage.

\eqref{egal:step_minimax_error} comes from Lemma \ref{lemma_minimax_value_bounds} and \eqref{egal:step_minimax_error_fact} from Fact \ref{fact_minimax_error}.
\end{proof}

\end{lemma}

\begin{lemma}[Optimism of the Policy computation for the player $p^-$]
\label{lemma_optimism_max_advantage}

For any \epoch{} $k$ for which Event $E$ defined by \eqref{regret_event} is False, the true model $M$ is in our plausible set $\mathcal{M}_k$, one player (named $p^-$) is receiving its maximum egalitarian advantage value in $M$, then for player $p^-$, we have:

\[\VUp_k^{p^-} \geq \VEg[]^{p^-}-2\epsilon_{t_k}\]

\end{lemma}

\begin{proof}
The proof is similar to the one of Lemma \ref{lemma_optimism_same_advantage} but using the fact that here, only $p^-$ is guaranteed to have the minimum egalitarian advantage in the true model. We have:

\begin{align}
	\VUp_k^{p^-} &= \tilde{V}^{p^-}(\tilde{\pi}_k)\\
	&\geq \tilde{V}^{p^-}(\tilde{\pi}_{k,\Eg}) -\epsilon_{t_k}\\
	&= \tilde{V}_+^{p^-}(\tilde{\pi}_{k,\Eg}) + \SVUp_k^{p^-} -\epsilon_{t_k}\\
	&\geq \min_j \tilde{V}_+^j(\pi_\Eg) + \SVUp_k^{p^-} -\epsilon_{t_k}\\
	&\geq \min_j V_+^j(\pi_\Eg) + \SVUp_k^{p^-} -\epsilon_{t_k}\\
	&= V_+^{p^-}(\pi_\Eg)+ \SVUp_k^{p^-} -\epsilon_{t_k}\\
	&= V_+^{p^-}(\pi_\Eg)+ \SV^{p^-} - \SV^{p^-} + \SVUp_k^{p^-} -\epsilon_{t_k}\\
	&= V^{p^-}(\pi_\Eg) + \SVUp_k^{p^-} - \SV^{p^-}-\epsilon_{t_k}\\
	&\geq V^{p^-}(\pi_\Eg) -2C_{t_k}(\tilde{\pi}^i_{\SV_k}, \hat{\pi}^{-i}_{\SV_k})-\epsilon_{t_k}\\
	&= \VEg[]^{p^-}-2\epsilon_{t_k}
\end{align}
\end{proof}

\begin{lemma}[Optimism of the Policy computation for player $p^+$]
\label{lemma_optimism_not_max_advantage}
For any \epoch{} $k$ for which the true model $M$ is in our plausible set $\mathcal{M}_k$, one player (named $p^-$, the other is named $p^+$) is receiving its maximum egalitarian advantage value in $M$ and Event $E$ defined by \eqref{regret_event} is False, then for player $p^{+}$, we have:
\[\VUp_k^{p^+} \geq \VEg[]^{p^+} -4\epsilon_{t_k}\]
\end{lemma}
\begin{proof}
	
First, we decompose the different cases that appear when Event $E$ \eqref{regret_event} is False. Then, we will treat each cases separately.

We have:

\begin{align}
	E\; \text{is False} &\implies E_1 \; \text{is False} \land E_2 \; \text{is False} \land E_3 \; \text{is False} \land \; \text{is False}\\
	&\implies \tilde{\pi}_k \ne \tilde{a}_{\SV_k} \land \tilde{\pi}_k \ne \tilde{a}_{k,\Eg} \land \paren*{a_* \in \tilde{\mathcal{A}}_{p^-} \lor 2C_{t_k}(\tilde{\pi}_{k,\Eg}) > \epsilon_{t_k} \lor \pi_{Eg} \ne a_{*} }\\
	&\quad\quad\;\;\land \paren*{a_{*} \notin \tilde{\mathcal{A}}_{p^-} \lor \tilde{\pi}_k \ne \tilde{a}_{p^-} \lor p^+ \notin \tilde{\mathcal{P}} \lor \pi_{Eg} \ne a_{*} \lor V_+^{p^+}(a_{*}) \leq V_+^{p^-}(a_{*}) +  2\epsilon_{t_k}}\notag\\
	&\implies \tilde{\pi}_k \ne \tilde{a}_{\SV_k} \land \tilde{\pi}_k \ne \tilde{a}_{k,\Eg} \land \paren*{a_* \in \tilde{\mathcal{A}}_{p^-} \lor 2C_{t_k}(\tilde{\pi}_{k,\Eg}) > \epsilon_{t_k} }\label{pplus:step:pplus}\\
	&\quad\quad\;\;\land \paren*{a_{*} \notin \tilde{\mathcal{A}}_{p^-} \lor \tilde{\pi}_k \ne \tilde{a}_{p^-} \lor p^+ \notin \tilde{\mathcal{P}} \lor V_+^{p^+}(a_{*}) \leq V_+^{p^-}(a_{*}) +  2\epsilon_{t_k}}\notag\\
	&\implies \tilde{\pi}_k \ne \tilde{a}_{\SV_k} \land \tilde{\pi}_k \ne \tilde{a}_{k,\Eg} \land a_* \in \tilde{\mathcal{A}}_{p^-} \label{pplus:step:egal}\\
	&\quad\quad\;\;\land \paren*{a_{*} \notin \tilde{\mathcal{A}}_{p^-} \lor \tilde{\pi}_k \ne \tilde{a}_{p^-} \lor p^+ \notin \tilde{\mathcal{P}} \lor V_+^{p^+}(a_{*}) \leq V_+^{p^-}(a_{*}) +  2\epsilon_{t_k}}\notag\\
	&\implies \tilde{\pi}_k \ne \tilde{a}_{\SV_k} \land \tilde{\pi}_k \ne \tilde{a}_{k,\Eg} \land a_* \in \tilde{\mathcal{A}}_{p^-} \label{pplus:step:best}\\
	&\quad\quad\;\;\land \paren*{ \tilde{\pi}_k \ne \tilde{a}_{p^-} \lor p^+ \notin \tilde{\mathcal{P}} \lor V_+^{p^+}(a_{*}) \leq V_+^{p^-}(a_{*}) +  2\epsilon_{t_k}}\notag
\end{align}

\eqref{pplus:step:pplus} comes from the assumption that $\pi_{Eg} = a_*$ since $p^-$ is receiving its maximum egalitarian advantage in $M$

\eqref{pplus:step:egal} comes from the fact that by construction of the Algorithm it is impossible for $\tilde{\pi}_k \ne \tilde{a}_{k,\Eg}$ and $2C_{t_k}(\tilde{\pi}_{k,\Eg}) > \epsilon_{t_k}$ to be True simultaneously.

\eqref{pplus:step:best} comes from the fact that is it not possible to have $a_* \in \tilde{\mathcal{A}}_{p^-}$ and $a_{*} \notin \tilde{\mathcal{A}}_{p^-}$ simultaneously.

Now let's decompose $\tilde{\pi}_k \ne \tilde{a}_{p^-} \lor p^+ \notin \tilde{\mathcal{P}}$ part of \eqref{pplus:step:best}

\begin{align}
\tilde{\pi}_k \ne \tilde{a}_{p^-} \lor p^+ \notin \tilde{\mathcal{P}} &\implies \paren*{\tilde{\pi}_k \ne \tilde{a}_{p^-} \land p^+ \in \tilde{\mathcal{P}}} \lor \paren*{\tilde{\pi}_k \ne \tilde{a}_{p^-} \land p^+ \notin \tilde{\mathcal{P}}} \lor p^+ \notin \tilde{\mathcal{P}}\\
&\implies \paren*{\tilde{\pi}_k \ne \tilde{a}_{p^-} \land p^+ \in \tilde{\mathcal{P}}} \lor p^+ \notin \tilde{\mathcal{P}}\label{pplus:step:simple}
\end{align}

Replacing \eqref{pplus:step:simple} into \eqref{pplus:step:best} gives us:

\begin{align}
	E\; \text{is False} &\implies \paren*{\tilde{\pi}_k \ne \tilde{a}_{\SV_k} \land \tilde{\pi}_k \ne \tilde{a}_{k,\Eg} \land a_* \in \tilde{\mathcal{A}}_{p^-} \land p^+ \notin \tilde{\mathcal{P}}}\notag\\
	&\quad\quad\;\;\lor \paren*{\tilde{\pi}_k \ne \tilde{a}_{\SV_k} \land \tilde{\pi}_k \ne \tilde{a}_{k,\Eg} \land a_* \in \tilde{\mathcal{A}}_{p^-} \land \tilde{\pi}_k \ne \tilde{a}_{p^-} \land p^+ \in \tilde{\mathcal{P}}}\label{pplus:cases}\\
	&\quad\quad\;\;\lor \paren*{\tilde{\pi}_k \ne \tilde{a}_{\SV_k} \land \tilde{\pi}_k \ne \tilde{a}_{k,\Eg} \land a_* \in \tilde{\mathcal{A}}_{p^-} \land V_+^{p^+}(a_{*}) \leq V_+^{p^-}(a_{*}) +  2\epsilon_{t_k}}\notag
\end{align}

We will now bound the EBS value of $p^+$ for each term of \eqref{pplus:cases} as a separate cases.
.

\subparagraph{Case: $\tilde{\pi}_k \ne \tilde{a}_{\SV_k} \land \tilde{\pi}_k \ne \tilde{a}_{k,\Eg} \land a_* \in \tilde{\mathcal{A}}_{p^-} \land V_+^{p^+}(a_{*}) \leq V_+^{p^-}(a_{*}) +  2\epsilon_{t_k}$}

We have:

\begin{align}
	\VUp_k^{p^+} &= \tilde{V}^{p^+}(\tilde{\pi}_k)\\
	&\geq \tilde{V}^{p^+}(\tilde{\pi}_{k,\Eg})-\epsilon_{t_k}\\
	&= \tilde{V}_+^{p^+}(\tilde{\pi}_{k,\Eg}) + \SVDown_k^{p^+} -\epsilon_{t_k}\\
	&\geq \min_j \tilde{V}_+^j(\pi_{\Eg}) + \SVDown_k^{p^+} -\epsilon_{t_k} \\
	&\geq \min_j V_+^j(\pi_{\Eg}) + \SVDown_k^{p^+} -\epsilon_{t_k} \\
	&= V_+^{p^-}(\pi_{\Eg})+ \SVDown_k^{p^+} -\epsilon_{t_k}\\
	&\geq V_+^{p^+}(\pi_{\Eg})+ \SVDown_k^{p^+} -3\epsilon_{t_k}\\
	&= V_+^{p^+}(\pi_{\Eg})+ \SV^{p^+} - \SV^{p^+} + \SVDown_k^{p^+} -3\epsilon_{t_k}\\
	&= V^{p^+}(\pi_{\Eg}) + \SVDown_k^{p^+} - \SV^{p^+}-3\epsilon_{t_k}\\
	&\geq V^{p^+}(\pi_{\Eg}) -2C_{t_k}(\tilde{\pi}^i_{\SV_k}, \hat{\pi}^{-i}_{\SV_k})-3\epsilon_{t_k}\\
	&= \VEg^{p^+}-4\epsilon_{t_k}\label{pplus:first_case}
\end{align}


\subparagraph{Case: $\tilde{\pi}_k \ne \tilde{a}_{\SV_k} \land \tilde{\pi}_k \ne \tilde{a}_{k,\Eg} \land a_* \in \tilde{\mathcal{A}}_{p^-} \land \tilde{\pi}_k \ne \tilde{a}_{p^-} \land p^+ \in \tilde{\mathcal{P}}$}

\begin{align}
\VUp_k^{p^+} &= \tilde{V}^{p^+}(\tilde{\pi}_k)\\
&\geq \tilde{V}^{p^+}(\tilde{a}_{p^+})\\
&= \tilde{V}_+^{p^+}(\tilde{a}_{p^+}) + \SVDown_k^{p^+}\\
&= \max_{a \in \tilde{\mathcal{A}}_{p^-}}\tilde{V}_+^{p^+}(a) + \SVDown_k^{p^+}\\
&\geq \tilde{V}_+^{p^+}(a_*) + \SVDown_k^{p^+}\\
&= \tilde{V}^{p^+}(a_*) -\SVDown_k^{p^+} + \SVDown_k^{p^+}\\
&\geq \VEg^{p^+}\label{pplus:second_case}
\end{align}

\subparagraph{Case: $\tilde{\pi}_k \ne \tilde{a}_{\SV_k} \land \tilde{\pi}_k \ne \tilde{a}_{k,\Eg} \land a_* \in \tilde{\mathcal{A}}_{p^-} \land p^+ \notin \tilde{\mathcal{P}}$}

\begin{align}
	\VUp_k^{p^+} &= \tilde{V}^{p^+}(\tilde{\pi}_k)\\
	&\geq \tilde{V}^{p^+}(\tilde{\pi}_{k,\Eg})-\epsilon_{t_k}\\
	&= \tilde{V}_+^{p^+}(\tilde{\pi}_{k,\Eg}) + \SVDown_k^{p^+} -\epsilon_{t_k}\\
	&\geq \max_{a \in \tilde{\mathcal{A}}_{p^-}} \tilde{V}^{p^+}_+(a) + \SVDown_k^{p^+} -\epsilon_{t_k} \label{pplus:step:last_case_not}\\
	&\geq\tilde{V}^{p^+}_+(a_*) + \SVDown_k^{p^+} -\epsilon_{t_k} \\
	&= \tilde{V}^{p^+}(a_*) -\SVDown_k^{p^+} + \SVDown_k^{p^+} -\epsilon_{t_k}\\
	&\geq V^{p^+}(\pi_{\Eg}) -\epsilon_{t_k}\\
	&= \VEg^{p^+}-\epsilon_{t_k}\label{pplus:third_case}
\end{align}

\eqref{pplus:step:last_case_not} is due to the fact that since $p^+ \notin \mathcal{P}$ we have: $\max_{a \in \tilde{\mathcal{A}}_{p^-}} \tilde{V}^{p^+}_+(a) \leq \VUp_+^{p^+}(\tilde{\pi}_{k,\Eg})$

Combining \eqref{pplus:first_case}, \eqref{pplus:second_case}, \eqref{pplus:third_case} leads to the statement of the Lemma.

\end{proof}

\section{On the Egalitarian Bargaining solution}

\subsection{Achievable values for both players}
\label{proof:fact:achievable_values}

\begin{fact*}[\ref{fact:achievable_values}]
\input{fact:achievable_values}
\end{fact*}
\begin{proof}
We start by showing that values for deterministic stationary policies exists and are unique, then we conclude by showing that any values can be achieved simultaneously for both player by a single stationary correlated-policy.

For any player $i$, the value as defined by Definition \ref{def:value} of any stationary policy exists and is unique \cite{puterman2014markov} since the game would be equivalent to a $1$-state Markov Decision Process. As a result, the value of deterministic stationary policies (i.e joint-actions) exists.

When player $i$ play with a deterministic stationary policy $a^i$ and player $-i$ plays  with a deterministic stationary policy $a^{-i}$, the values  for  the  two  players  can  be  visualized  as a  point $x = \paren*{V^i(a^i, a^{-i}), V^{-i}(a^i, a^{-i})} = (x^i, x^{-i})$ in a two-dimensional space.

Following \cite{nash1950bargaining}, we consider the set of all pairs of (values for) deterministic policies $X = \{\paren*{V^i(a^i, a^{-i}), V^{-i}(a^i, a^{-i})}\; \forall a^{i} \in \mathcal{A}^i,a^{-i} \in \mathcal{A}^{-i} \}$ for the two players. All the points $x \in X$ can be achieved as value for the two players in the repeated game, simply by repeatedly playing the corresponding joint-action.

Consider the convex hull $\mathcal{C}$ of the set of points $x \in X$. This means that any point in the convex hull can be expressed as a weighted linear combination of the points $x \in X$ where the weights sum up to 1. Those weights can thus be seen as probabilities which allows us to affirm that any point in the convex hull can be achieved as values for the two players in the repeated game (by playing the corresponding stationary policy with the weight as probabilities). On the other side, any achievable values for the two players belongs to the convex since hull which follows from the definition of convex hull. In conclusion, the convex hull represents exactly the set of all achievable values for the two players. And since any point in the convex hull is achievable by a stationary policy, this concludes our proof.
\end{proof}

\subsection{Existence and Uniqueness of the EBS value for stationary policies}
\label{proof:fact:egalitarian_unique}

\begin{fact*}[\ref{fact:egalitarian_unique}]
\input{fact:egalitarian_unique}
\end{fact*}

\begin{proof}
\cite{imai1983individual} proves that the EBS value as defined in Definition \ref{definition:equilibrium} always exists and is unique for any bargaining problem that is convex, closed, of non-empty Pareto frontier and non degenerate (i.e there exists a point greater or equal than the disagreement point).

To conclude the proof of this fact, it is then enough to prove that we have a bargaining problem satisfying those properties. We consider the bargaining problem induced by the repeated game.

Here our disagreement point is the \minimax{} value.

From the proof of Fact \ref{fact:achievable_values}, we can see that the repeated game with stochastic rewards can be replaced by another one with deterministic rewards corresponding to the values of joint-actions. As a result, the \minimax{} value exists and is unique. Also, there always exists a unique (one-stage) Nash Equilibrium which is greater in value than the \minimax{} value of both players \cite{nash1951non}. This means that player can always get their \minimax{} value. So our bargaining problem is non-degenerate.

Finally, using the same convex hull as in the proof of Fact \ref{fact:achievable_values}, we can see that the set of achievable values is convex. This set is also closed since the joint-actions are finite and the rewards are bounded. And there always exists one Pareto efficient policy (Any policy achieving the maximum value for one player). This concludes the proof.
\end{proof}

\subsection{On the form of an EBS policy}
\label{proof:proposition:egalitarian_form}
\begin{proposition*}[\ref{proposition:egalitarian_form}]
\input{proposition:egalitarian_form}
\end{proposition*}

\begin{proof}
Let's recall that the EBS value maximize the minimum possible for any player and as a result if we have a value where the minimum advantage can't be improved anymore, then we have the EBS value provided that we also maximize the value of the second player if possible.
 
Now let's consider the convex hull defined in the proof of Fact \ref{fact:achievable_values}. The egalitarian point will be found on the outer boundary of the convex hull -- the minimum value of any internal point can be increased by moving to a point above it and to the right and higher minimum means higher EBS value.  This implies that the egalitarian point can be expressed by a weight vector $w$ that has non-zero weight on only one or two $x \in X$, since the convex hull is a two-dimensional polygonal region bounded by line segments.

\end{proof}

\subsection{Finding an EBS policy}
\label{proof:proposition:egalitarian_equal}
\begin{proposition*}[\ref{proposition:egalitarian_equal}]
\input{proposition:egalitarian_equal}
\end{proposition*}

\begin{proof}
	From proposition \ref{proposition:egalitarian_form} we can achieve the EBS value by combining at most two deterministic stationary policies. We will prove this proposition (\ref{proposition:egalitarian_equal}) for any two possible deterministic stationary policies (by considering a repeated game with only the corresponding joint-actions available), which immediately means that the proposition \ref{proposition:egalitarian_equal} is also true for the EBS value in the full repeated game.
	
	Consider any two deterministic stationary policy of advantage values ($(x^1_1, x^2_1), (x^1_2, x^2_2)$). We will now show how to compute the weight $w = \argmax_{w} \min_{i \in \{1,2\}} w*x^i_1 + (1-w)x^{i}_2$.
	
	\paragraph{Case 1:} $x^1_1 \leq x^2_1$ and $x^1_2 \leq x^2_2$.
	This basically means that the advantage value of player $2$ is always higher or equal than that of the player $1$. So the minimum is maximized by playing the policy maximizing the value of player $1$. So, $w = 0$ and we have a single deterministic stationary policy where the player with the lowest ideal advantage receives it.
	
	\paragraph{Case 2:} $x^1_1 \geq x^2_1$ and $x^1_2 \geq x^2_2$.
	This is essentially \emph{Case 1} with the role of player $1$ and $2$ exchanged. Here $w = 1$.
	
	If both \emph{Case 1} and \emph{Case 2} do not hold, it means that for the first policy, one player receives an advantage value strictly greater than that of the other player while the situation is reversed for the second policy. Without loss of generality we can assume this player is $1$ (if this is not the case, we can simply switch the id of the policy) which leads to \emph{Case 3}.
	
	\paragraph{Case 3:} $x^1_1 > x^2_1$ and $x^1_2 < x^2_2$
	In this case, the optimal $w$ is such that $w = \frac{x^2_2 - x^1_2}{(x^1_1- x^1_2) + (x^2_2 - x^2_1) }$. This weight $w$ is clearly between the open interval $]0,1[$. This means that we have exactly two distinct policies. Plugging in the weight shows that the advantage value of both player is the same, which completes the proof.
\end{proof}

\section{Regret analysis for the safe policy against arbitrary opponents}
\label{proof:theo:safe_regret}

\begin{theorem*}[\ref{theo:safe_regret}]
\input{theo:safe_regret}
\end{theorem*}

\begin{proof}
The proof is similar to the one for Theorem \ref{theo:egalitarian_upper_bound}. However, here we don't have to deal with the event $E$ defined in \eqref{regret_event} which is thus taken to always be \emph{False}.

Also, we are always optimistic (against the true \minimax{} value when the true model $M$ is within our plausible set $\mathcal{M}_k$) by playing policy $\tilde{\pi}_{\SV}^i$ computed in \eqref{algo:safe_minimax_policy}. Indeed, for any opponent policy $\pi^{-i}$,

\begin{align}
\tilde{\pi}_{SV}^i &= \argmax_{\pi^i} \min_{\pi^{-i}} \tilde{V}^i(\pi^i, \pi^{-i})\\
\tilde{V}(\tilde{\pi}_{SV}^i, \pi^{-i})&=\max_{\pi^i} \min_{\pi^{-i}} \tilde{V}^i(\pi^i, \pi^{-i})\\
&\geq \max_{\pi^i} \min_{\pi^{-i}} V^i(\pi^i, \pi^{-i})\\
&= \SV^i \label{eq:safe_is_optimistic}
\end{align}

Note that this is not a contradiction to the upper bound in Lemma \ref{lemma_minimax_value_bounds} since the $\SVUp^i$ mentioned in Lemma \ref{lemma_minimax_value_bounds} is computed using the value $\hat{V}$ (rather than $\tilde{V}$) lower than the empirical values as shown by \eqref{algo:safe_minimax_value}.

As a result, the corresponding $\epsilon_{t_k}$ used by step \eqref{eq:Delta_optimism} in the proof of Theorem \ref{theo:egalitarian_upper_bound} is 0.

Finally, the arguments justifying step \eqref{eq:Delta_egalitarian} in the proof of Theorem \ref{theo:egalitarian_upper_bound} does not hold anymore since we are now playing a completely random policy. Instead, we can bound the deviation of $\sum_{k=1}^{m}\sum_{a \in \mathcal{A}} N_{k}(a) \paren*{\VUp_k^i - \rUp^i(a)}$ using Chernoff-bound. Combining those remarks into the proof of Theorem \ref{theo:egalitarian_upper_bound} leads to statement of Theorem \ref{theo:safe_regret}.
\end{proof}

\section{Proof of the Lower bounds in Theorem \ref{theo:lower_bound}}
\label{proof:theo:lower_bound}

\begin{table}[H]
	\centering
		\begin{tabulary}{\linewidth}{ |C|C|C|C|C| }
			\hline
			& $a^2_1$ & $a^2_2$  & $\cdots$ & $a^2_{\size{\mathcal{A}^2}}$ \\ \hline
			$a^1_1$	& $(0.5, 1)$ & $(0.5, 0.5)$ & $\cdots$ & $(0.5,0.5)$ \\ \hline
			$a^1_2$	& $(0.5[+\epsilon], 0.5[+\epsilon])$ & $(0.5, 0.5)$ & $\cdots$ & $(0.5,0.5)$ \\ \hline
			$\vdots$	& $\vdots$ & $\vdots$  & $\cdots$ & $(0.5,0.5)$ \\ \hline
			$a^1_{\size{\mathcal{A}^1}}$	& $(0.5, 0.5)$ & $(0.5, 0.5)$  & $\cdots$ & $(0.5,0.5)$ \\ \hline
		\end{tabulary}
		
		\caption{Lower bounds example. The rewards are generated from a Bernoulli distribution whose parameter is specified in the table. The first value in parentheses is the one for the first player while the other is for the second player. Here, $\epsilon$ is a small constant defined in the proof.}
		\label{table:lower_bound_egalitarian}
\end{table}

\begin{theorem*}[\ref{theo:lower_bound}]
	\input{theo:lower_bound}
\end{theorem*}

\begin{proof}
	
	The proof is inspired by the one for bandits in Theorem 6.11 of \cite{cesa2006prediction}.
	First, we prove the theorem for deterministic stationary algorithms. The general case then follows by a simple argument. The main idea of the proof is to show that there exists a repeated game such that for any algorithm the expected regret is large (where the expectation is understood with respect to the random rewards from the repeated game). 
    
    Let the repeated games be as follows: The first player has $\size{\mathcal{A}^1}$ actions while the second player has $\size{\mathcal{A}^2}$ actions. At each round, the players' rewards are generated independently from a Bernoulli distribution whose expectation depends on the joint-action as follows: 
    For one action pair $a_*=(a^1_i, a^2_j)$, the rewards of the first player has expectation $\frac{1}{2}$ while the second player receives $1$. All the others joint-actions yield rewards with expectation $\frac{1}{2}$ to both players. However, we also selects uniformly at random one joint-action $a$ from the set of all joint-actions expect $a_*$ and with probability $\frac{1}{2}$, we switch the expectation of that action from $\frac{1}{2}$ to $\frac{1}{2} + \epsilon$ for both players. This selection can be thought as having a random variable $Z$ that selects $a_*$ with probability $\frac{1}{2}$ and, the remaining joint-actions with probability $\frac{1}{2(\size{\mathcal{A}}-1)}$
 
    Let's denote $\pmb{r}^T$ the sequence of all generated rewards and $\pmb{\mathcal{R}}^T$, the set of all possible sequence up to round $T$ Also, let $G^i_{., T}$ be the sum of rewards obtained by an algorithm for player $i$ up to \round{} $T$. Then 
      for any (non-randomized) algorithm we have:
      
      \begin{align}
      \sup_{\pmb{r}^T \in \pmb{\mathcal{R}}^T}( \GEg[T]^i - \GLambda[T]^i ) \geq \E\left[ (\GEg[T]^i - \GLambda[T]^i) \right] \forall i \in \{1,2\}\label{lower_bound_regret_vs_expectation}
      \end{align}
	
	where the expectation on the right-hand side is with respect to the random variables $r_t(a)$. Thus it suffices to bound, from above, the expected regret for the randomly chosen rewards. 
	
	First Observe that for player $i$:
	
	\begin{align}
			\E\left[\GEg[T]^i\right] &= \sum_{a \in \mathcal{A}} \Prob[Z = a] \E\left[\GEg[T]^i | Z=a\right]\\		
			&= \frac{1}{2A-2} \sum_{a \in \mathcal{A}\backslash a_* } \E\left[\GEg[T]^i | Z=a\right] \notag \\
			& \quad + \frac{1}{2} \E\left[\GEg[T]^i | Z=a_*\right]\\
			&= \frac{T}{2}  \VEgR[Z=a_*]^i + \frac{T}{2}  \VEgR[Z\ne a_*]^i
		\end{align}
		
			Computing the optimal EBS values for each player gives:
				
				$\VEgR[Z\ne a_*]^1 = \frac{1}{2} + \epsilon$,
				 $\VEgR[Z\ne a_*]^2 = \frac{1}{2} + \epsilon$,
				  $\VEgR[Z = a_*]^1 = \frac{1}{2}$,$ \VEgR[Z = a_*]^2 = 1$
				
			As a result, we can conclude about the optimal Egalitarian gain for both agents as:
			
			\begin{align}
			\E\left[\GEg[T]^1\right] &= \frac{T}{2} + \frac{T\epsilon}{2}\label{lower_bound_opt_first_agent}\\
			\E\left[\GEg[T]^2\right]&= \frac{3T}{4} + \frac{T\epsilon}{2}\label{lower_bound_opt_second_agent}
			\end{align}

	Now, we need to bound, from above the expected cumulative reward $\E \GLambda[T]^i$ for an arbitrary algorithm. To this end, fix a (deterministic stationary) algorithm and let $I_t$ denote the joint-action it chooses at round $t$. Clearly, $I_t$ is determined by the sequence of rewards $\pmb{r} = \{r_{I_1, 1}, \ldots r_{I_{t-1}, t-1}\}$. Also, let
	$T_{a} = \sum_{t=1}^{T} \Id_{I_t = a}$ be the number of times action $a$ is played by the algorithm. Then,
	writing $\E_a \text{for} \E[. | Z=a]$, we may write
	
	\begin{align}
			\E \GLambda[T]^i &= \frac{1}{2A-2} \sum_{a \in \mathcal{A}\backslash a_*} \E_a \GLambda[T]^i + \frac{1}{2} \E_{a_*}\GLambda[T]^i\\
			&= \frac{1}{2A-2} \sum_{a \in \mathcal{A}\backslash a_*} \sum_{t=1}^{T} \sum_{a' \in \mathcal{A}} \E_a r^i_{a',t} \E_a \Id_{I_t = a'}\notag\\
			&\quad + \frac{1}{2} \sum_{t=1}^{T} \sum_{a' \in \mathcal{A}} \E_{a_*} r^i_{a',t} \E_{a_*} \Id_{I_t = a'}\label{eq:lower_bound_bianchi}\\
			&= \frac{1}{2A-2} \sum_{a \in \mathcal{A}\backslash a_*} \sum_{t=1}^{T} \sum_{a' \in \mathcal{A}\backslash \setof{a,a_*}} \E_a r^i_{a',t} \E_a \Id_{I_t = a'}\notag\\
			&\quad +  \frac{1}{2A-2} \sum_{a \in \mathcal{A}\backslash a_*} \sum_{t=1}^{T} \E_a r^i_{a,t} \E_a \Id_{I_t = a}\notag\\
			&\quad +  \frac{1}{2A-2} \sum_{a \in \mathcal{A}\backslash a_*} \sum_{t=1}^{T} \E_a r^i_{a_*,t} \E_a \Id_{I_t = a_*}\\
			&\quad + \frac{1}{2} \sum_{t=1}^{T} \sum_{a' \in \mathcal{A}\backslash a_*} \E_{a_*} r^i_{a',t} \E_{a_*} \Id_{I_t = a'}\notag\\
			&\quad + \frac{1}{2} \sum_{t=1}^{T}  \E_{a_*} r^i_{a_*,t} \E_{a_*} \Id_{I_t = a_*}\notag
		\end{align}
		
		where Line \eqref{eq:lower_bound_bianchi} follows directly from \cite{cesa2006prediction}.
		
		So now we have
			
			\begin{align}
				\E G^1_{\Lambda, T} &= \frac{T}{2} + \frac{\epsilon}{2A-2}\sum_{a \in \mathcal{A}\backslash a_*} \E_a T_a\label{eq:gain_first_exact}\\
				\E G^2_{\Lambda, T} &= \frac{T}{2} + \frac{1}{4} \E_{a_*} T_{a_*} + \frac{\epsilon}{2A-2}\sum_{a \in \mathcal{A}\backslash a_*} \E_a T_a \notag \\
				&\quad + \frac{0.5}{2A-2}\sum_{a \in \mathcal{A}\backslash a_*} \E_a T_{a_*}\label{eq:gain_second_exact}
			\end{align}
	
	Next, we apply Lemma \ref{lemma:distance_good_bad_action} to $T_a$ which is a function of the rewards sequence $\pmb{r}$ since the actions of the algorithm $\Lambda$ are determined by the past rewards. Since $T_a \in [0, T]$
	
	\[\E_{a} [T_a] \leq \E_{a_*} [T_a] + \frac{T}{2} \sqrt{-\E_{a_*}[T_a](\ln2)  \ln(1-4\epsilon^2)} \]
	
	Now letting $x = \sum_{a \in \mathcal{A}\backslash a_*} \E_{a*}[T_a]$ and using the fact that $\sum_{a \in \mathcal{A}\backslash a_*} \sqrt{x} \leq \sqrt{x(A-1)}$, we have:
	
	\begin{align}
	\label{eq:sum_a_plus_bound}
	\sum_{a \in \mathcal{A}\backslash a_*} \E_{a} [T_a] \leq x + \frac{T}{2} \sqrt{x(A-1)  \ln(1-4\epsilon^2)\ln\frac{1}{2}}
	\end{align}
	
	Combining \eqref{eq:sum_a_plus_bound} with \eqref{eq:gain_first_exact} we can bound the gain of the first agent as:
	
	\begin{align}
	\E G^1_{\Lambda, T} &\leq \frac{T}{2} + \frac{\epsilon}{2A-2}\paren*{x + \frac{T}{2} \sqrt{x(A-1)  \ln(1-4\epsilon^2)\ln \frac{1}{2}}}\label{upper_bound_gain_first_agent}
	\end{align}
	
	Similarly, combining \eqref{eq:sum_a_plus_bound} with \eqref{eq:gain_second_exact} and the fact that $T_{a_*} + \sum_{a \in \mathcal{A}\backslash a_*} T_{a} = T$, we can bound the gain of the second agent as:
	
	\begin{align}
	\E G^2_{\Lambda, T} &\leq \frac{\epsilon}{2A-2}\paren*{x + \frac{T}{2} \sqrt{x(A-1)  \ln(1-4\epsilon^2)\ln \frac{1}{2}}}\notag\\
	&\quad + \frac{3T}{4} - \frac{1}{4}x + \frac{1}{4(A-1)}x\label{upper_bound_gain_second_agent}
	\end{align}
	
	We can now derive the lower bound for regret of the first agent by combining \eqref{upper_bound_gain_first_agent}, \eqref{lower_bound_regret_vs_expectation}, \eqref{lower_bound_opt_first_agent}
	
	\begin{align}
	\regret^1_{T, \Lambda} &\geq \frac{T\epsilon}{2} - \frac{\epsilon}{2A-2}\paren*{x + \frac{T}{2} \sqrt{x(A-1)  \ln(1-4\epsilon^2)\ln \frac{1}{2}}} 
	\end{align}

	Similarly for the second agent by combining \eqref{upper_bound_gain_second_agent}, \eqref{lower_bound_regret_vs_expectation}, \eqref{lower_bound_opt_second_agent}
	
	\begin{align}
	\regret^2_{T, \Lambda} &\geq \regret^1_{T, \Lambda} + \frac{1}{4}x - \frac{1}{4(A-1)}x  
	\end{align}

	We can now derive the overall regret of the algorithm (equal to the maximum of the regret of both agents) by looking for the $x$ leading to the smallest maximum. Picking $\epsilon = \min\setof{A^{1/3}T^{-1/3}, \frac{\sqrt{0.43}}{2}}$ and using the fact that $-\ln(1-z) \leq z + \frac{z^2}{2} + \frac{z^3}{2}$ for $z \in [0, 0.43]$ give the lower bound of $\Omega\paren*{\frac{T^{2/3}A^{1/3}}{4}}$
	 
	This concludes the proof for deterministic stationary algorithms. The extension to randomized algorithms follows the same argument as in \cite{cesa2006prediction}.	
\end{proof}

	\begin{lemma}
	\label{lemma:distance_good_bad_action}
	Let $f: \setof{(0,0), (0,1), (1,0), (1,1)}^T \to [0, M]$ be any function defined on rewards sequences $\pmb{r}^T$. Then for any action $a \in \mathcal{A}\backslash a_*$,
	
	\begin{align*}
	\E_{a}[f(\pmb{r}^T)] &\leq  \E_{a_*}[f(\pmb{r}^T)] +\\ 
	&\quad\frac{M}{2} \sqrt{-\E_{a_*}[f(\pmb{r}^T)](\ln2)  \ln(1-4\epsilon^2)}
	\end{align*}

	\end{lemma}
	
	\begin{proof}
	Similarly to Lemma A.1 in \cite{auer2002nonstochastic} we have
	
	\begin{align}
	\E_{a}[f(\pmb{r}^T) - \E_{a_*}[f(\pmb{r}^T)] &\leq \frac{M}{2} \sqrt{(2\ln2)  \kl{\Prob_{a_*}}{\Prob_{a}}}\label{lemma:expectation}	
	\end{align}
	
	Computing the KL-divergence similarly to Lemma A.1 in \cite{auer2002nonstochastic} leads to:
	
	\begin{align}
	\kl{\Prob_{a_*}}{\Prob_{a}} &= \sum_{t=1}^{T} \Prob_{a_*}\{I_t=a\}\kl*{\frac{1}{2}}{\frac{1}{2} + \epsilon}\\
	&= \E_{a_*}[f(\pmb{r}^T)] \paren*{-\frac{1}{2}\ln(1-4\epsilon^2)}\label{lemma:kl}	
	\end{align}
	
	The lemma follows by combining \eqref{lemma:expectation} and \eqref{lemma:kl}

	\end{proof}

\section{Previously Known results}

\begin{lemma}[Chernoff-Hoeffding bound \cite{hoeffding1963probability}]
\label{hoeffding}
Let $X_1, X_2 \ldots X_n$ be random variables with common range $[0,1]$ and such that $\E\left[X_t \mid X_1, \ldots X_{t-1}\right] = \mu$. Then for all $\epsilon \geq 0$

\[\Prob\curly*{\sum_{i=1}^{n} X_i \geq n\mu + \epsilon} \leq \exp\paren*{-\frac{2\epsilon^2}{n}} \quad \text{and} \quad \Prob\curly*{\sum_{i=1}^{n} X_i \leq n\mu - \epsilon} \leq \exp\paren*{-\frac{2\epsilon^2}{n}}\]

\end{lemma}

\begin{fact}[Union Bound also known as Boole's inequality]
\label{union_bound}
For a countable set of events $A_1, A_2, \ldots$ we have:
\[\Prob\curly*{\bigcup_i A_i}  \leq \sum_{i} \Prob(A_i)\]
\end{fact}


\section{Algorithms}
\label{sec:appendix_algo}

\subsection{Finding an EBS policy for a  game with known rewards distribution}
For any game $M$ with rewards $r$ this can be done using \eqref{eq:solution_ebs} with $\rUp_+$ replaced by $r_+$


\subsection{Communication protocol}
It is also important for the players to communicate since the policies of the players might need to be correlated to play the same joint-action. This communication\footnote{It is possible to remove this assumption and uses more cryptographically robust synchronization protocol with minimal communication. However, this is out of scope of this paper.} is done through lexicographical ordering the policies using the unique actions identifier and player identifier assumed to be shared before the game start by both players. When a policy involve playing multiple actions with different probabilities, players simply play actions such that their empirical probability of play is as close as possible to the true policy probability. This is explained more formally in the function $\Call{Play}$ of Algorithm \ref{algo:egalitarian_optimism}. 

\begin{algorithm}[H]
	\caption{Generic Optimism in the face of uncertainty}
	\label{algo:egalitarian_optimism}
	\begin{algorithmic}
	
	\State \textbf{Initialization:} 
	\ParState{$N_k(a)$ denotes the number of \rounds{} action $a$ has been played in episode $k$ --- $N_k$ the number of \rounds{} episode $k$ has lasted --- $t_k$ the number of rounds played up to episode $k$ --- $N_{t_k}(a)$ the number of \rounds{} action $a$ has been played up to round $t_k$ --- $\bar{r}_t^i(a)$ the empirical average rewards of player $i$ for action $a$ at round $t$.}
	\State Let $t \gets 1$
	\State Set $N_k, N_{k}(a), N_{t_{k}}(a)$ to zero for all $a \in \mathcal{A}$.
	
	\Statex
	\For{episodes $k=1, 2, \ldots$}
	\State $t_k \gets t$
	\State $N_{t_{k+1}}(a) \gets N_{t_k}(a) \quad \forall a$
	
	\State $\rUp^i_{k}(a) = \bar{r}_k^i(a) + C_k(a),
	\qquad
	\rDown^i_k(a) = \bar{r}_k^i(a) - C_k(a) \quad \forall a, i$ with $C_k$ as in \eqref{plausible_set}.
	
	\State $\tilde{\pi}_{k} \gets \Call{OptimisticEgalitarianPolicy}{\rEmp_t, \rUp_k, \rDown_k}$
	
	\Statex
	\State \textbf{Execute policy $\tilde{\pi}_{k}$}:
	
	\Do{}
	
	\State Let $a_t \gets \Call{Play}$, play it and observe $r_t$
	
	\State $N_k \gets N_k +1$ $\quad N_{k}(a_t) \gets N_{k}(a_t) + 1$ $\quad N_{t_{k+1}}(a_t) \gets N_{t_{k+1}}(a_t) + 1$
	
	\State Updated $\bar{r}_t(a_t)$
	
	\State $t \gets t + 1$
	
	\DoWhile{$N_k(a_t) \leq \max\{1, N_{t_k}(a)\}$}

	\EndFor{}
		
%
%
%
%
%
%
%
%
%
%
%
%
%
%
%
%
		\Statex
		\Statex
		
		\Function{Play}{}
		
		\ParState{Let $a_t$ the action $a$ that minimizes $\abs*{\tilde{\pi}_k(a)-\frac{N_k(a)}{N_k}}$ 
		}
		\ParState {Ties are broken in favor of the player with the lowest, then in favor of the lexicographically  smallest action.}
		\Return $a_t$
		\EndFunction
		
		\Statex

	\end{algorithmic}

\end{algorithm}

\end{document}